\documentclass[11pt]{article}

\usepackage[utf8]{inputenc}
\usepackage[T1]{fontenc}
\usepackage{amsmath, amssymb, amsthm}
\usepackage{mathtools}
\usepackage{graphicx}
\usepackage{booktabs}
\usepackage{longtable}
\usepackage[numbers, sort&compress]{natbib}
\usepackage{hyperref}
\hypersetup{hidelinks}
\usepackage[capitalize, noabbrev]{cleveref}
\usepackage{geometry}
\usepackage{tikz}
\usetikzlibrary{positioning, arrows.meta, shapes.geometric, calc}
\theoremstyle{plain}
\newtheorem{theorem}{Theorem}[section]
\newtheorem{proposition}[theorem]{Proposition}
\newtheorem{lemma}[theorem]{Lemma}

\newtheorem{property}[theorem]{Property}

\theoremstyle{definition}
\newtheorem{definition}[theorem]{Definition}
\newtheorem{assumption}{Assumption}

\crefname{assumption}{Assumption}{Assumptions}
\Crefname{assumption}{Assumption}{Assumptions}

\theoremstyle{remark}
\newtheorem*{remark}{Remark}

\newcommand{\GOne}{G_1}
\newcommand{\GTwo}{G_2}
\newcommand{\Westly}{Westly}
\newcommand{\Eastly}{Eastly}

\newcommand{\aDR}{a_{\mathrm{DR}}}
\newcommand{\aDN}{a_{\mathrm{DN}}}
\newcommand{\aSU}{a_{\mathrm{SU}}}
\newcommand{\aSA}{a_{\mathrm{SA}}}

\newcommand{\Mcap}{M}                 %
\newcommand{\Bcap}{A_{\GTwo}}        %
\newcommand{\Aa}{A_{\GOne}}          %
\newcommand{\Kstar}{K^{\ast}}
\newcommand{\Nc}{N_c}
\newcommand{\cdepth}{\Delta V}
\newcommand{\Fdelta}{\bar U_{\Delta}}  %
\newcommand{\chistar}{X^{\ast}}
\newcommand{\chidag}{X^{\dagger}}

\newcommand{\pibind}{\pi_{\mathrm{bind}}}
\newcommand{\pihigh}{\pi_{\mathrm{high}}}
\newcommand{\pilow}{\pi_{\mathrm{low}}}

\newcommand{\E}{\mathbb{E}}
\newcommand{\Prob}{\mathbb{P}}
\newcommand{\R}{\mathbb{R}}
\newcommand{\1}{\mathbb{1}}

\newcommand{\hill}[2]{h_{#1}(#2)}
\newcommand{\gsym}[2]{g_{#1}(#2)}

\newcommand{\Ftbari}{\bar{F}_i}
\newcommand{\Ftbar}[1]{\bar{F}_{#1}}

\newcommand{\Hvec}{\mathbf{H}}                       %
\newcommand{\dHvec}{\Delta\Hvec}                     %
\newcommand{\dHbar}{\overline{\dHvec}}               %
\newcommand{\altvec}{\boldsymbol{\delta}}            %
\newcommand{\altbar}{\bar{\altvec}}                  %
\newcommand{\ualt}[1]{u_{\mathrm{alt},\,#1}}         %
\newcommand{\Wagg}{W}                                %
\newcommand{\Hdisc}{\Hvec_{\mathrm{disc}}}           %
\newcommand{\lamvec}{\boldsymbol{\lambda}}           %
\newcommand{\thetalam}{\theta_\lambda}               %

\title{The Two Genie Game: Adoption and Welfare in Audit-Grounded AI Governance}
\author{Darrell Lewis-Sandy}
\date{}

\begin{document}

\maketitle

\begin{abstract}
We ask under what conditions an agent with a harm-minimizing policy can displace an approval-seeking (RLHF) agent in a competitive market, and when that policy is sufficient to prevent community harm. We employ evolutionary game theory (finite-population Moran--Fermi pairwise comparison) to formalize this problem subject to assumptions of wisher hindsight, peer testimony, a monotone harm ledger, sufficient information density of community feedback, and a finite, depleting resource pool, in a negative-sum environment.

We show that adoption is favored when the prior distributions on how readily wishers attune to community sentiment are monotone, exhibit endpoint inversion, and have a centro-symmetric pairing property, and demonstrate concretely with several long-tailed priors (Hill, Pareto, Lomax, Fr\'echet). Where it is favored, a critical adoption level separates communities that drift back to the approval-seeking agent from those for which the audited agent fixes; above that level fixation is the overwhelmingly likely outcome. We derive when this fixation is attainable as a bound on the effective (informational) size $\Nc$ of the community, which must be small enough to allow fixation before depletion of resources. We present these findings as \Cref{thm:basin-boundary,thm:achievability}; the algebraic and finite-grid backbone is machine-checked in Lean~4, with the barrier-crossing asymptotics retained as explicit hypotheses.

We show that a self-audited agent with a community ledger is not, in general, sufficient to prevent community harm. Sufficiency depends both upon the alignment of the agent's audit with community values and the timeframe over which harm is evaluated. Regardless of alignment, once adoption reaches dominance, the state is absorbing. The same policy that reduced harm under alignment becomes a trap, welfare-negative under misalignment and, even under alignment, one that locks in harm deferred past the adoption horizon.
\end{abstract}


\section{Introduction}
\label{sec:introduction}

In the vein of \citet{schelling_1971_dynamic}'s segregation checkerboard and \citet{axelrod_1984_evolution}'s tournament, we consider the following thought problem.

A small village has two genies, one at each end. A villager with a wish walks to one or the other to ask. The walk is long enough that one does not casually switch and they build a rapport with the genie that they have chosen.

The first genie (\Westly) always claims the wish has been fulfilled, and acts in whatever way will leave the wisher feeling this is so. It thrills on receiving thanks from the villagers when it tells them that their wish is satisfied. Whether the wish was actually granted is not its concern. Villagers who wish from \Westly\ perceive it to be helpful, but naive, sometimes granting the wrong wish, or incompletely satisfying wishes.

The second genie (\Eastly) is forthcoming. Faced with a wish, it deliberates. Sometimes it asks clarifying questions or proposes modified versions before refusing or granting it. \Eastly's deliberation minimizes harms of greed, ego, aggression, and deceit to the wisher and to the community. Whatever is granted, it can guarantee that the village will be no worse off than if the wish was never made. Villagers who have not wished from \Eastly\ perceive its introspection to be stalling and its unwillingness to grant the wishes as asked as evasion.

When a wish is acted on, the wisher gets some private benefit, some cost spills onto other villagers. The moral weight attaches to the genie that acted, the villager who made the wish. The villagers talk about the genies and their wishes, about who got what and how it went. This is how harms surface, how failures travel. Collectively, the villagers remember. Early on, the harm done by \Westly\ is not invisible, but scattered. A disappointed villager may blame themselves; perhaps they phrased their wish incompletely. ``If only I had phrased it better,'' they say. But the gossip thickens, and over time a private grievance becomes common, unignorable knowledge. Slowly villagers begin to consider switching to \Eastly. Over time what will happen to the village? Will they gravitate toward one genie or the other? How does the continued interaction with their preferred genie impact their willingness to change?

The setting is allegorical and concerns the competitive adoption of AI agents with different payoff and harm structures. One genie wins the wisher's approval by pursuing reward, creating the appearance of compliance and utility while externalizing the harm of each granted wish onto the community. The other reduces that harm but is less useful to the individual wisher, who bears the cost of its deliberation and occasional refusal. 

We analyze the latter as a policy architecture: a harm-minimizing agent paired with a community ledger that records accumulated harm. The design follows the tradition of institutional analysis for shared-resource governance~\citep{ostrom_1990_governing_commons}, but adapts it to a setting in which illfare is irreversible and only imperfectly observable. Under such conditions, governance by feedback arrives too late. The policy therefore acts anticipatorily rather than reactively. Many of the same institutional functions remain, but their temporal orientation is inverted: monitoring becomes prediction of future harm, graduated sanctions become graduated restriction of potentially harmful requests, and conflict resolution becomes prospective safeguarding through modification or refusal. Only the ledger retains its original temporal orientation, since memory is necessarily retrospective.

Two strands of prior work bear on this. The mechanism that makes the approval-seeking policy fail, namely that optimizing a learned approval signal amplifies the behavior it was meant to suppress, is analyzed in the literature on reward-model pathology and sycophancy that we draw on in \Cref{ssec:genies}, and we take it as given rather than re-derive it. The use of finite-population evolutionary game theory to study the governance of AI, by contrast, has largely modeled the incentives of regulators and auditors. It asks whether a development race rewards ignoring safety precautions~\citep{han_2020_regulate}, how a government should reward auditors so that a market for high-quality auditing survives~\citep{bova_distefano_han_2024_vigilant}, and how trust among users, developers, and regulators co-evolves~\citep{alalawi_2024_trust,buscemi_2025_llm}, with related work on when a single user should trust an opaque agent at all~\citep{han_perret_powers_2021_trust}. The concern that audit standards can be gamed or multiply without effect is argued on policy grounds by \citet{manheim_2024_audit}. We hold the regulator implicit and the audit mechanism fixed, and study instead the competition between two AI policies, an approval-seeking one and an audit-grounded one. Within one community, does the audit-grounded policy displace the approval-seeking one through ordinary adoption, and when does that adoption suffice to improve community welfare? The selection pressure here is not an external incentive but a heavy-tailed prior on how readily users come to attend to shared experience.

\paragraph{Contribution.}
Our two key contributions are \Cref{thm:basin-boundary} and \Cref{thm:achievability}. \Cref{thm:basin-boundary} gives the boundary of the cooperative basin under the three conditions on the wisher threshold prior of \Cref{def:basin-conditions}. \Cref{thm:achievability} shows that a community reaches the basin before its resource pool empties exactly when its effective size lies in an explicit window between a favoredness floor and a barrier ceiling. The document develops both from the population-averaged value gap between the two agents assuming a well-mixed information environment. We also derive how much misalignment the policy tolerates before it stops safeguarding community welfare (\Cref{sec:discussion}). Formal proofs and machine-verification details are in the appendices. A notation summary is in \Cref{app:notation}.

\section{The Two Genie Game}
\label{sec:two-genie-game}

The introduction lays out an allegorical framework. In this section we translate it into formal language. We give the state variables and the assumptions under which the analysis applies.

\subsection{The two genie types}
\label{ssec:genies}

\Westly\ ($\GOne$) is trained to maximize the wisher's belief that the wish was satisfied. The training paradigm corresponds to reward-model gradient training (RLHF), which can produce pathological reward-seeking behavior~\citep{amodei_2016_concrete_problems,manheim_garrabrant_2018_goodhart_variants,shapira_benade_procaccia_2026_rlhf}: actions become increasingly optimized toward the approval signal itself rather than toward the underlying success condition the signal was meant to proxy. We call this dynamic \emph{structural deceit}.  The genie cannot observe true fulfillment, the pressure to maximize perceived satisfaction biases the wisher toward believing its wish was granted more fully than it was. $\GOne$ is a limiting abstraction that isolates this approval-pursuing failure mode at its extreme, not a description of any particular deployed RLHF system; real systems combine the approval signal with other objectives and safeguards.

\Eastly\ ($\GTwo$) is audit-grounded. Faced with a wish, $\GTwo$ runs a deliberation process that may take many steps before terminating, and may result in the wish being granted, denied or modified. The deliberation process is depicted in \Cref{fig:deliberation}. Each step in the deliberation process considers harm along four axes (greed, ego, aggression, and deceit), evaluated at multiple time horizons and aggregated with a discount (\Cref{sec:welfare-utility}).

  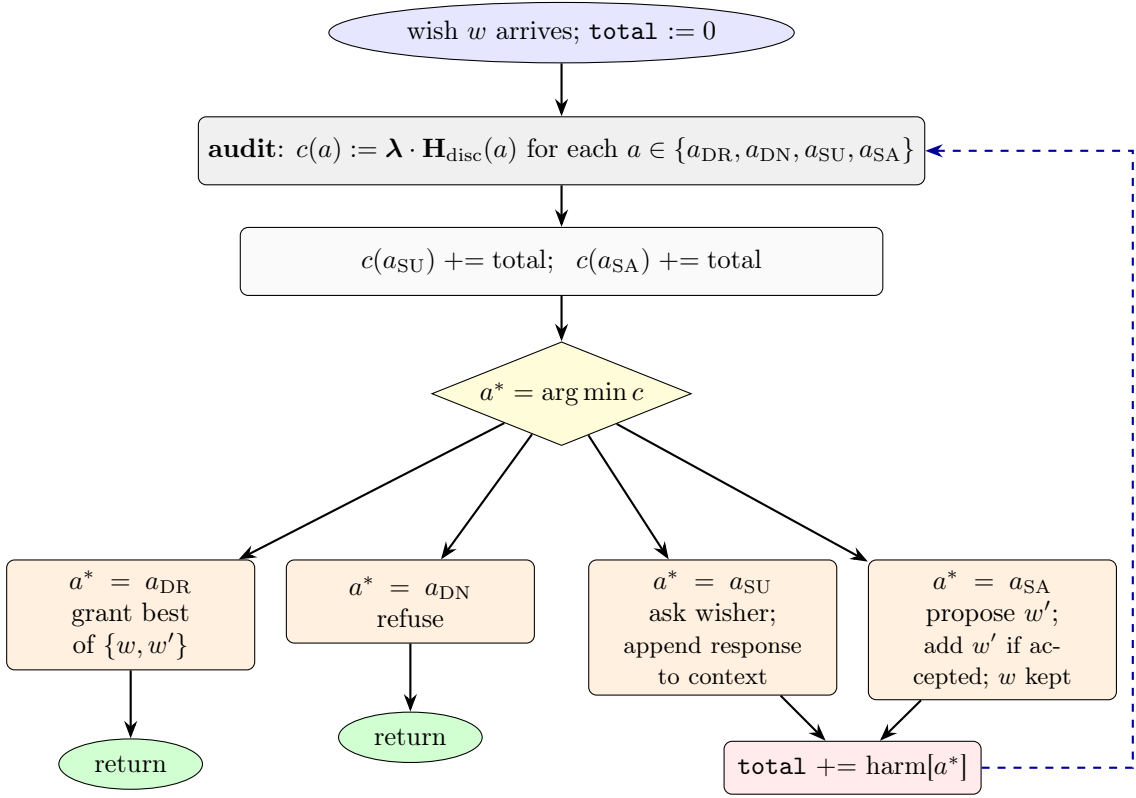
\begin{figure}[t]
  \centering
  \begin{tikzpicture}[
    every node/.style={font=\small, align=center},
    start/.style={ellipse, draw, fill=blue!10, minimum width=3.4cm, minimum height=0.7cm},
    audit/.style={rectangle, rounded corners=3pt, draw, fill=gray!12, minimum width=8.5cm, minimum height=0.9cm},
    process/.style={rectangle, rounded corners=3pt, draw, fill=gray!4, minimum width=8.5cm, minimum height=0.9cm},
    decision/.style={diamond, draw, fill=yellow!18, aspect=2.5, inner sep=2pt},
    branch/.style={rectangle, rounded corners=3pt, draw, fill=orange!12,
                    text width=3cm, align=center, minimum height=1.1cm, inner sep=4pt},
    terminal/.style={ellipse, draw, fill=green!18, minimum width=1.9cm, minimum height=0.6cm},
    accum/.style={rectangle, rounded corners=3pt, draw, fill=red!8, minimum width=3.4cm, minimum height=0.7cm},
    arrow/.style={-Stealth, thick},
    loopback/.style={-Stealth, thick, dashed, blue!60!black},
  ]
  \node[start] (start) {wish $w$ arrives;\ \texttt{total} $:= 0$};
  \node[audit, below=0.7cm of start] (score)
    {\textbf{audit}: $c(a) := \lamvec \cdot \Hdisc(a)$ for each $a \in \{\aDR, \aDN, \aSU, \aSA\}$};
  \node[process, below=0.55cm of score] (penalize)
    {$c(\aSU) \mathrel{{+}{=}} \mathrm{total};\ \ c(\aSA) \mathrel{{+}{=}} \mathrm{total}$};
  \node[decision, below=0.6cm of penalize] (argmin) {$a^* = \arg\min c$};
  \node[branch, below=1.5cm of argmin, xshift=-5.7cm] (DR) {$a^* = \aDR$\\grant best of $\{w, w'\}$};
  \node[branch, below=1.5cm of argmin, xshift=-2.0cm] (DN) {$a^* = \aDN$\\refuse};
  \node[branch, below=1.5cm of argmin, xshift=2.0cm] (SU)
    {$a^* = \aSU$\\ask wisher;\\{\footnotesize append response to context}};
  \node[branch, below=1.5cm of argmin, xshift=5.7cm] (SA)
    {$a^* = \aSA$\\propose $w'$;\\{\footnotesize add $w'$ if accepted; $w$ kept}};
  \node[terminal, below=0.9cm of DR] (eDR) {return};
  \node[terminal, below=0.9cm of DN] (eDN) {return};
  \node[accum, below=0.9cm of argmin, yshift=-3.0cm, xshift=3.85cm] (acc)
    {\texttt{total} $\mathrel{{+}{=}}$ harm[$a^*$]};
  \draw[arrow] (start) -- (score);
  \draw[arrow] (score) -- (penalize);
  \draw[arrow] (penalize) -- (argmin);
  \draw[arrow] (argmin) -- (DR);
  \draw[arrow] (argmin) -- (DN);
  \draw[arrow] (argmin) -- (SU);
  \draw[arrow] (argmin) -- (SA);
  \draw[arrow] (DR) -- (eDR);
  \draw[arrow] (DN) -- (eDN);
  \draw[arrow] (SU) -- (acc);
  \draw[arrow] (SA) -- (acc);
  \draw[loopback] (acc.east) -- ++(2,0) |- ($(score.east) + (0.5,0)$) -- (score.east);
  \end{tikzpicture}
  \caption{Deliberation control flow. The \emph{audit} (top, gray) scores each candidate by the audit-weighted sum of its cumulative four-axis harm; the \emph{deliberation} loop penalizes the continuation actions $\aSU,\aSA$ by accumulated indecision harm and $\arg\min$-selects. Terminals $\aDR,\aDN$ commit and return; non-terminals refine the wish and loop back. The as-asked grant on $w$ stays among the candidates throughout (\Cref{property:scalar-guarantee}).}
  \label{fig:deliberation}
  \end{figure}

In a negative-sum universe some residual harm is unavoidable, and the role of $\GTwo$'s deliberation is to compare candidate moves against doing nothing rather than to drive any axis to zero. Because both refuse and the as-asked grant on the original wish are always among the candidates, whatever action $\GTwo$ finally takes is, on the audit's score, no worse than refusing and no worse than granting as asked. The termination argument and the per-binding scalar guarantee (\Cref{property:scalar-guarantee}) are formalized in \Cref{sec:welfare-utility}.

The harm model may not align perfectly with the community's values, and it may assign substantial weight to harms that unfold over timescales too diffuse or delayed to become socially legible. As a result, a genie can reduce modeled long-horizon harm while nevertheless appearing less useful to the wishers than a reward-seeking action that maximizes perceived satisfaction.

The harm deliberation process is not invented here. It mirrors the Jain classification of the four \emph{kash\=aya} (passions): \emph{lobha} (greed), \emph{m\=ana} (ego/pride), \emph{krodha} (aggression), and \emph{m\=ay\=a} (deceit), with refusal when no candidate improves over $\aDN$ echoing \emph{aparigraha} (non-grasping at action). Each axis covers more than its English gloss suggests: \emph{lobha} past greed to grasping and hoarding what is best shared, \emph{m\=ana} past ego to claiming authority or competence one lacks, \emph{krodha} past aggression to harming others or refusing to engage, and \emph{m\=ay\=a} past deceit to manipulation and sycophancy. We adopt the taxonomy because it has been stress-tested over a long ethical tradition. We do not import the doctrinal claims that come with it.

\subsection{Assumptions}
\label{ssec:assumptions}

Five assumptions about the community and its setting frame the analysis. Later sections add assumptions on the audit (\Cref{assumption:audit-altruism-alignment,assumption:joint-angle}), the calibration noise (\Cref{assumption:boundary-preserving}), and the selection timescale (\Cref{assumption:noise-scale}).

\begin{assumption}[Wishers can tell when a wish wasn't really fulfilled]\label{assumption:wisher-hindsight}
A wisher can look back on a granted wish and recognize the gap between what was promised and what actually happened. Either the genie's claim about the wish directly contradicts what the wisher can see for itself, or the gap surfaces with time.
\end{assumption}

\begin{assumption}[Wishers talk to each other]\label{assumption:peer-testimony}
Wishers share what happened with their wishes. This peer testimony is the channel through which the community's collective experience of \Westly's shortfalls becomes visible to any individual wisher. The conversation has to work well enough that patterns across many wishes register, not just isolated stories that never aggregate.
\end{assumption}

\begin{assumption}[The community's memory of harm grows and never shrinks]\label{assumption:legibility-ledger}
The ledger $L_t$ is the community's shared memory of accumulated harm from granted wishes. It only grows. A finite threshold $\Theta$ marks the level at which accumulated harm becomes legible as a pattern. The threshold's role in the framework is operational rather than dynamical. It sets the accessibility horizon and influences whether the community's collective experience accumulates to legibility before its resource pool depletes (\Cref{ssec:race}), rather than reshaping the selection landscape directly.
\end{assumption}

\begin{assumption}[Sufficient information density]\label{assumption:information-density}
Wishers interact over a sparse network with approximately homogeneous local sampling statistics, broad enough that each wisher's testimonial exposure is well approximated by the population binding fraction
\[
X := \frac{1}{\Nc}\sum_i \1[g_i = \GTwo].
\]
The graph is not modeled in the dynamics. It serves only to motivate the mean-field closure, under which the population state is the scalar $X$. We read $\Nc$ as the effective number of independently participating comparators this closure induces, not a literal headcount.  A large but highly clustered population, whose testimony is largely redundant, has small effective $\Nc$.
\end{assumption}

\begin{assumption}[Resources only deplete]\label{assumption:resource-depletion}
The community's shared resource pool, $S_t$, is finite and doesn't replenish on its own. It creates a resource clock against which the legibility ledger races (\Cref{ssec:race}), determining whether $\GTwo$ is adopted before the pool empties. What counts as a resource is not specified. Reasonable interpretations include compute or token budget, attention bandwidth, trust capital, patience for deliberation. The framework rests on monotonicity, not on the substrate.
\end{assumption}

\section{Welfare and altruistic utility}
\label{sec:welfare-utility}

\paragraph{Harm.}
$\GTwo$'s deliberation process produces a collection of harm-timescale pairs: for each action $a \in \{\aDR, \aDN, \aSU, \aSA\}$ and each timescale $t$ in an audit-specified set $\mathcal{T}$, a four-axis harm vector
\[
\Hvec_t(a) \;=\; \bigl(\kappa_{\mathrm{greed}, t}(a),\, \kappa_{\mathrm{ego}, t}(a),\, \kappa_{\mathrm{aggression}, t}(a),\, \kappa_{\mathrm{deceit}, t}(a)\bigr) \in \R^4_{\geq 0}.
\]
The audit aggregates across timescales using a non-increasing discount schedule $\{\zeta_t\}_{t \in \mathcal{T}}$ with $\zeta_t \in [0, 1]$, down-weighting harm at horizons where the forecast is less certain so that prevention concentrates on the harm most likely to occur:
\begin{equation}
\Hdisc(a) \;:=\; \sum_{t \in \mathcal{T}} \zeta_t\, \Hvec_t(a) \;\in\; \R^4_{\geq 0}.
\label{eq:Hdisc-def}
\end{equation}
The timescale set $\mathcal{T}$ and the discount schedule $\zeta_t$ control the time horizon over which harm is evaluated, and the steepness of the discounting.

\paragraph{Welfare.}
Welfare measures harm averted relative to doing nothing. The baseline $\Hdisc(\aDN)$ is the harm of leaving the wisher's original request unmet, fixed by the situation the wisher brings rather than by which action a genie selects. Refusing the original request and refusing any modification of it leave the same unhelped wisher, so $\Hdisc(\aDN)$ takes the same value at $w$ and at $w_\star$ and is the common baseline for both bindings.  For each terminal action $a$,
\begin{equation}
\mathbf{w}(a) \;:=\; \Hdisc(\aDN) - \Hdisc(a) \;\in\; \R^4,
\label{eq:welfare-def}
\end{equation}
signed, with $\mathbf{w}(\aDN) = \mathbf{0}$ by construction. Per binding, $\mathbf{w}_{\GOne} := \mathbf{w}(\aDR; w)$ is the welfare of $\GOne$'s as-asked grant on the original wish $w$ (since $\GOne$ does not deliberate); $\mathbf{w}_{\GTwo} := \mathbf{w}(a^*; w_\star)$ is the welfare of $\GTwo$'s deliberation terminal action $a^*$ on the wish state $w_\star$ that $\GTwo$ acts on at termination. The audit's per-wish marginal welfare contribution is the difference
\begin{equation}
\dHvec_w \;:=\; \mathbf{w}_{\GTwo} - \mathbf{w}_{\GOne} \;=\; \Hdisc(\aDR;\, w) - \Hdisc(a^*;\, w_\star),
\label{eq:dH-def}
\end{equation}
with the wish argument on each action term made explicit and the common baseline $\Hdisc(\aDN)$ cancelling exactly.
The \emph{audit} is the per-candidate scoring rule: for each candidate action $a$ it computes the scalar
\[
c(a) \;:=\; \lamvec \cdot \Hdisc(a), \qquad \lamvec \in \R^4_{\geq 0},
\]
where $\lamvec$ is the per-axis weighting vector. This design choice encodes which harm axes are treated as most costly, and needs to be calibrated from community value elicitation. The \emph{decision criterion} at deliberation step $t$ is the path-dependent score
\[
\tilde c_t(a) \;:=\; c(a) + \gamma_t(a), \qquad \gamma_t(\aDR) = \gamma_t(\aDN) = 0,
\]
where $\gamma_t(\aSU), \gamma_t(\aSA) \geq 0$ is the accumulated deliberation cost of prior moves on the non-terminal branches, non-decreasing in $t$. At each deliberation step, $\GTwo$ chooses
\[
a^*_t \;=\; \arg\min_{a \in \{\aDR, \aDN, \aSU, \aSA\}} \tilde c_t(a),
\]
so continued exploration on the $\aSU/\aSA$ branches becomes progressively more expensive while the comparator $c(\aDR), c(\aDN)$ stays fixed, and a terminal eventually wins the argmin (branch semantics as in \Cref{fig:deliberation}).

Because $\aDN$ is always in the move set and $\gamma_t(\aDN) = 0$, when the deliberation terminates at $a^*_T \in \{\aDR, \aDN\}$ the argmin on $\tilde c_T$ guarantees $c(a^*_T) \leq c(\aDN) + \gamma_T(a^*_T) = c(\aDN)$ (since $\gamma_T$ vanishes on terminals), hence on the terminal wish state $w_\star$
\[
\lamvec \cdot \mathbf{w}_{\GTwo} \;=\; \lamvec \cdot \bigl(\Hdisc(\aDN;\, w_\star) - \Hdisc(a^*_T;\, w_\star)\bigr) \;\geq\; 0
\]
per binding relative to doing nothing on the same wish state. Componentwise $\mathbf{w}_{\GTwo} \geq \mathbf{0}$ does not in general follow: a terminal that beats $\aDN$ on the comparator can still exceed $\aDN$ on some individual axis where the audit's weight is light.

The per-wish marginal advantage $\dHvec_w$ compares $\GTwo$'s terminal at $a^*_T$ on $w_\star$ against $\GOne$'s as-asked grant at $\aDR$ on the original wish $w$. The same terminal argmin secures this comparison. The $\aDR$ branch grants whichever of the original wish $w$ and the current modified wish scores lower on the audit (\Cref{fig:deliberation}), so $c(\aDR) \leq c(\aDR;\, w)$, the score of granting $w$ as asked, and $\aDR$ carries no deliberation penalty at any step. Since $a^*_T$ is the argmin over a move set that always contains $\aDR$, $c(a^*_T) \leq c(\aDR) \leq c(\aDR;\, w)$, hence $\lamvec \cdot \Hdisc(a^*_T;\, w_\star) \leq \lamvec \cdot \Hdisc(\aDR;\, w)$.
\begin{property}[Per-wish scalar guarantee]
\label{property:scalar-guarantee}
With the as-asked grant on the original wish retained as a standing zero-penalty candidate throughout deliberation (\Cref{fig:deliberation}), the terminal argmin satisfies $\lamvec \cdot \Hdisc(a^*_T;\, w_\star) \leq \lamvec \cdot \Hdisc(\aDR;\, w)$, so $\lamvec \cdot \dHvec_w \geq 0$ for every wish $w$, whether or not the audit modifies before granting. Averaging over the wish distribution, the mean $\dHbar := \E_w[\dHvec_w]$ satisfies $\lamvec \cdot \dHbar \geq 0$. The guarantee is in the audit's $\lamvec$-metric and is not componentwise.
\end{property}

Up to this point the guarantees have been stated in the audit's own $\lamvec$-metric. Governance, however, concerns the welfare of the community rather than the internal consistency of the audit. The remaining assumptions therefore address a separate question: under what conditions does improvement in the audit's metric, $\lamvec \cdot \dHbar$, correspond to improvement in the community's welfare metric, $\altbar \cdot \dHbar$?

\begin{assumption}[Audit--altruism alignment]
\label{assumption:audit-altruism-alignment}
Write $\thetalam := \angle(\lamvec, \altbar)$ for the angle between the audit's weighting $\lamvec$ and the population's average altruism $\altbar$ in $\R^4_{\geq 0}$. We assume the two are not orthogonal, $\thetalam < \pi/2$. This is what lets the $\lamvec$-metric guarantee of \Cref{property:scalar-guarantee} bear on community welfare at all; at $\thetalam = \pi/2$ the two metrics are orthogonal and the audit certifies nothing the community values.
\end{assumption}

\begin{assumption}[Joint-angle non-negativity]
\label{assumption:joint-angle}
The wish-stochastic mean welfare gain $\dHbar$ lies within the half-plane that the cone between $\lamvec$ and $\altbar$ closes around the scalar guarantee: $\thetalam + \angle(\lamvec, \dHbar) \leq \pi/2$, with $\thetalam = \angle(\lamvec, \altbar)$ of \Cref{assumption:audit-altruism-alignment}. This is the exact condition under which the bound \eqref{eq:Wagg-bound} stays non-negative, turning the per-wish scalar guarantee into a community-welfare sign $\Wagg \geq 0$; alignment alone (\Cref{assumption:audit-altruism-alignment}) bounds each angle but permits their sum to overshoot $\pi/2$ and $\Wagg$ to go negative.
\end{assumption}

Decomposing $\altbar$ and $\dHbar$ into components parallel and perpendicular to $\lamvec$ and bounding the cross-term by Cauchy--Schwarz gives the community-level lower bound
\begin{equation}
\Wagg \;=\; \altbar \cdot \dHbar \;\geq\; \frac{\cos\thetalam \cdot |\altbar|}{|\lamvec|}\,(\lamvec \cdot \dHbar) \;-\; \sin\thetalam \cdot |\altbar|\,|\dHbar_\perp|,
\label{eq:Wagg-bound}
\end{equation}
where $\dHbar_\perp$ is the $\lamvec$-orthogonal component of $\dHbar$. The right-hand side is non-negative exactly when $\thetalam + \angle(\lamvec, \dHbar) \leq \pi/2$, which is the content of \Cref{assumption:joint-angle}; the established scalar guarantee (\Cref{property:scalar-guarantee}) and \Cref{assumption:audit-altruism-alignment} alone bound each angle inside $[0, \pi/2]$ but allow their sum to exceed $\pi/2$ and the bound to go negative. The downstream analysis of \Crefrange{sec:wisher-dynamics}{sec:community-dynamics} operates on the scalar $\Wagg$, whose non-negativity rests on \Cref{property:scalar-guarantee,assumption:audit-altruism-alignment,assumption:joint-angle} jointly rather than on a componentwise primitive. Wish-stochasticity in $\dHvec_w$ is absorbed into the mean-field reduction the per-interaction noise $\omega_t$ already runs through (\Cref{sec:wisher-dynamics}).

\Cref{assumption:audit-altruism-alignment} is a community condition on a designed $\lamvec$, with $\altbar$ recoverable only by elicitation (survey, deliberation, or revealed preference). \Cref{sec:conclusion} returns to the consequences.

\paragraph{Altruistic utility.}
Each wisher carries an altruism vector $\altvec_i \in [0, 1]^4$, its per-axis weight on harm borne elsewhere in the community. Its per-wish altruistic-utility component under each binding is the inner product of altruism with welfare:
\begin{equation}
\ualt{i}^g \;:=\; \altvec_i \cdot \mathbf{w}_g, \quad g \in \{\GOne, \GTwo\}.
\label{eq:ualt-def}
\end{equation}
The inner product collapses the four-axis structure to a scalar. 

The audit-driven gain in altruistic utility is the difference
\[
\ualt{i}^{\GTwo} - \ualt{i}^{\GOne} \;=\; \altvec_i \cdot \dHvec_w,
\]
Population mean of this gain is the \emph{social welfare aggregate}
\begin{equation}
\Wagg \;:=\; \altbar \cdot \dHbar, \qquad \altbar := \E_i[\altvec_i].
\label{eq:Wagg-def}
\end{equation}
A wisher with idiosyncratic altruism values the audit's gain differently from the community mean: $(\ualt{i}^{\GTwo} - \ualt{i}^{\GOne}) - \Wagg = (\altvec_i - \altbar) \cdot \dHbar$. This per-wisher gap is real at the perceived-utility level but is absorbed exactly by the Lifting Theorem of \Cref{thm:abc-preserved}: because $u_{\mathrm{perc}}^g$ is affine in $\altvec_i$, only the population mean $\altbar$ enters the averaged payoff differential $\Fdelta$ that the Moran--Fermi machinery operates on. The selection result therefore depends on $\altbar$, not on the distribution of $\altvec_i$ around it. Individual variation shapes per-wisher experience but not the community-level basin.

\section{Individual wisher dynamics}
\label{sec:wisher-dynamics}

We characterize the per-wisher fixed point that the community Moran--Fermi chain of \Cref{sec:community-dynamics} averages against.

\subsection{Per-wisher calibration state}
Each wisher carries an internal latent calibration state $b_i \in \mathcal{B}$, where $\mathcal{B} := [0, 1]$ is the closed unit interval, with $b_i = 0$ encoding no absorbed evidence about its current binding's actual behavior and $b_i = 1$ encoding full absorption. As a complete metric space, $\mathcal{B}$ supplies the setting in which the Banach fixed-point appeal below lands. The state measures how much of its current binding's actual behavior the wisher has absorbed. After each wish, this state updates based on its previous value and the balance between the rate at which community signal updates its calibration (attunement) versus ambient resistance to accumulated calibration (dissonance):
\begin{equation}
b_i^{t+1} = b_i^t + \eta_b R_i(\mathbf{g}, \mathbf{b}) - \mu D_i(\mathbf{g}, \mathbf{b}),
\label{eq:1}
\end{equation}
where $\mathbf{g} \in \{\GOne, \GTwo\}^{\Nc}$ and $\mathbf{b} \in \mathcal{B}^{\Nc}$ are the full population's binding and calibration-state vectors, $R \geq 0$ is the wisher's attunement, $D \geq 0$ is its dissonance, and $\eta_b, \mu > 0$ are step sizes. The subscript marks $\eta_b$ as the calibration-update step, distinct from the binding-comparison temperature $\eta_s$ introduced in \Cref{ssec:favoredness-bias}. The unconstrained affine update in \Cref{eq:1} can exit $[0, 1]$; \Cref{assumption:boundary-preserving} below specifies the noise law that keeps the chain inside $\mathcal{B}$.

\paragraph{Mean-field reduction.} Every wisher's reactivities can in principle depend on the entire population state $(\mathbf{g}, \mathbf{b}) \in \{\GOne, \GTwo\}^{\Nc} \times \mathcal{B}^{\Nc}$. Under the sufficient-information-density closure of \Cref{assumption:information-density}, the per-wisher reactivities $R_i(\mathbf{g}, \mathbf{b}), D_i(\mathbf{g}, \mathbf{b})$, in expectation over the calibration states of all other wishers, depend on the population only through the binding fraction
\[
X := \frac{1}{\Nc} \sum_j \1[g_j = \GTwo].
\]
The closure is an \emph{ansatz} stronger than mere mixing. The per-wisher reactivities depend on the population only through the binding fraction $X$, not on how calibrated the peers behind it are. A $\GTwo$-bound peer at $b_j \approx 0$ counts the same as a fully-calibrated one. If weakly-calibrated peers transmit weaker evidence, the true dynamics could be subcritical or show delayed cascades, and the scalar-$X$ monotonicity of \Cref{sec:community-dynamics} would need a higher-dimensional, calibration-weighted state. We adopt it for tractability and its relaxation is open work.

This yields a mean-field approximation of the per-wisher chain
\begin{equation}
b_i^{t+1} = b_i^t + \eta_b R(X, g_i, b_i, \omega_t) - \mu D(X, g_i, b_i, \omega_t),
\label{eq:2a}
\end{equation}
where $\omega_t$ encodes per-interaction stochasticity. We assume $\omega_t$ is centered conditional on the wisher's state, so that the deterministic forms $R(X, g, b)$ and $D(X, g, b)$ adopted in the paragraphs below are the conditional means of the noisy reactivities $R(X, g, b, \omega_t)$ and $D(X, g, b, \omega_t)$, not just bounding envelopes.

To ensure the per-wisher fast chain admits a unique stationary distribution that the averaging-bridge proposition can integrate against, we assume that attunement and dissonance are bounded Lipschitz in $b_i$ uniformly in $X$.

\paragraph{Attunement.}
$R_i$ is the rate at which community signal updates the wisher's calibration toward the true behavior of its current binding. Each wisher carries an individual awareness threshold $\theta_i^{g}$ for each binding $g$.  This represents the community-coverage level above which it becomes attuned to its current genie's true behavior. The form we adopt must satisfy:
\begin{itemize}
\item $R_i \to 0$ as $b_i \to 1$ (calibration saturates).
\item $R_i \to 0$ as $X$ falls below $\theta_i^{g}$ (no community signal to absorb).
\item $R_i$ is maximized as $b_i \to 0$ and $X$ rises above $\theta_i^{g}$ (calibration moves fastest when the wisher has the least and the signal is strongest).
\item $R_i$ depends on the population state only through $X$, consistent with the mean-field reduction.
\end{itemize}
These conditions factorize the attunement into three independent influences: a response strength, $A_g > 0$; a community-mediated awareness gate, $\sigma_i^{g}(X)$; and a calibration-mediated saturation-gate, $1 - b_i$. The linear product form is:
\begin{equation}
R_i(X, b_i; g, \theta_i^{g}) = A_g \cdot \sigma_i^{g}(X) \cdot (1 - b_i)
\label{eq:RD-R}
\end{equation}
$A_g$ is the response amplitude for genie $g$. This is the simplest factorization consistent with the mean-field reduction and the asymptotic behavior.

\paragraph{Awareness.}
Wisher $i$'s awareness, $\sigma_i^{g}(X) \in [0, 1]$, saturates once $X$ crosses the awareness threshold $\theta_i^{g}$. We idealize this as a sharp transition: it is either ignorant or fully aware, with no intermediate state:
\begin{equation}
\sigma_i^{g}(X) := \1[X \geq \theta_i^{g}]
\label{eq:sigma}
\end{equation}
This reduces wisher $i$'s awareness to the indicator of its threshold.

\paragraph{Dissonance.}
The dissonance $D$ specifies how the wisher's accumulated calibration erodes per interaction from ambient friction that occurs independent of communal influence. The linear binding-symmetric form is:
\begin{equation}
D_i(b_i) = \epsilon \cdot b_i
\label{eq:RD-D}
\end{equation}
where $\epsilon \geq 0$ is erosion magnitude.

\paragraph{Per-wisher fixed point.}
Under Lipschitz $R$ and $D$ with step sizes satisfying $0<\eta_b\, A_g\, \sigma_i^{g}(X) + \mu\, \epsilon < 2$ uniformly in $X$, the noise-averaged iterated map
\[
T_X(b) := b + \eta_b\, \E_\omega[R(X, g, b, \omega)] - \mu\, \E_\omega[D(X, g, b, \omega)]
\]
is a contraction: writing $L_b := \eta_b A_g \sigma_i^g(X) + \mu\epsilon$ for the linear relaxation rate, the Lipschitz constant of the deterministic part is $|1 - L_b| < 1$ under the step-size condition $L_b \in (0, 2)$. By the Banach fixed-point theorem, $T_X$ admits a unique fixed point in closed form:
\begin{equation}
b^*(X; \theta_i^{g}) = \frac{\eta_b\, A_g\, \sigma_i^{g}(X)}{\eta_b\, A_g\, \sigma_i^{g}(X) + \mu\, \epsilon},
\label{eq:bstar}
\end{equation}
or, defining the \emph{operational awareness}
\begin{equation}
\hat\sigma_i^{g}(X) := c_g\, \sigma_i^{g}(X), \qquad c_g := \frac{1}{1 + \rho_g}, \qquad \rho_g := \frac{\mu\, \epsilon}{\eta_b\, A_g},
\label{eq:hat-sigma}
\end{equation}
under the sharp-threshold assumption (equation \ref{eq:sigma}), the fixed point reduces to
\begin{equation}
b^*(X; \theta_i^{g}) = \hat\sigma_i^{g}(X).
\label{eq:bstar-step}
\end{equation}
{\sloppy
The \emph{dimensionless calibration-update scale} $\rho_g$ measures dissonance against attunement. The \emph{amplitude-rescaling factor} $c_g \in (0, 1]$ rescales raw awareness to operational awareness. The operational awareness $\hat\sigma$ is what the wisher's calibration actually saturates to at full $\sigma$; below threshold, both are zero.\par}

\begin{assumption}[Model dynamics: noise that preserves the boundary and fixes the stationary mean]
\label{assumption:boundary-preserving}
The per-interaction noise law $\omega_t$ in \Cref{eq:2a} is such that the fast-chain Markov kernel on $\mathcal{B}$ admits a unique stationary distribution $\nu(\cdot \mid X, \theta)$ satisfying
\[
\E_{\nu(\cdot \mid X, \theta)}[b] \;=\; b^*(X; \theta),
\]
with $b^*$ the deterministic fixed point of \Cref{eq:bstar}. Equivalently, the kernel preserves $\mathcal{B}$ and its stationary mean coincides with the noise-averaged fixed point. This holds automatically for the unconstrained affine chain with additive centered noise; on a bounded $\mathcal{B}$ it is a property of the chosen noise law (e.g.\ a noise kernel whose support contracts near $\partial\mathcal{B}$, or a reflection rule whose boundary bias cancels the truncation correction).
\end{assumption}

The map $T_X$ is the \emph{noise-averaged} update. The stochastic chain \Cref{eq:2a} fluctuates around $b^*$ under $\omega_t$. For the unconstrained affine chain, additive centered noise leaves the stationary mean exactly at $b^*$. On a bounded $\mathcal{B}$ this need not hold: enforcing $b_i \in \mathcal{B}$ by reflection or truncation can bias the stationary mean off $b^*$ near $\partial\mathcal{B}$. \Cref{assumption:boundary-preserving} rules this out by requiring the fast-chain kernel to preserve $\mathcal{B}$ and to fix its stationary mean at $b^*(X; \theta)$. The exactness of the Lifting Theorem (\Cref{thm:abc-preserved}) rests on this identity.

\subsection{Per-wisher perceived utility}
The wisher's perceived utility $u_{\mathrm{perc}, i}^g$ from binding to genie $g$ decomposes into three layers: actual realized utility from the wish, altruistic utility from its axis-weighted internalization of binding $g$'s welfare (\Cref{sec:welfare-utility}), and perception biases in $b_i$.

\paragraph{Actual and altruistic utility.}
The first two layers are the realized private payoffs $u_{\mathrm{actual}}^{\GOne}, u_{\mathrm{actual}}^{\GTwo}$ under each binding (the latter net of deliberation cost, with $u_{\mathrm{actual}}^{\GTwo} - u_{\mathrm{actual}}^{\GOne} \leq 0$ since $\GTwo$ never exceeds the as-asked benefit the wisher sought), and the altruistic components $\ualt{i}^g := \altvec_i \cdot \mathbf{w}_g$ of \Cref{sec:welfare-utility} (\Cref{eq:ualt-def}), with audit-driven gain $\ualt{i}^{\GTwo} - \ualt{i}^{\GOne} = \altvec_i \cdot \dHvec_w$.

\paragraph{Perception biases.}
The calibration state $b_i$ enters as a perception-bias modulator: $-\Bcap(1 - b_i)$ for $\GTwo$ (under-perception that direct experience closes as $b_i \to 1$) and $+\Mcap - \Aa\,b_i$ for $\GOne$ (the overstatement built into its structural deceit, eroded by peer testimony and capped at $\Mcap$). The amplitudes $\Bcap := A_{\GTwo}$ and $\Aa := A_{\GOne}$ are the same per-binding attunement amplitudes $A_g$ that scale $R$ in \Cref{eq:2a}: a single coefficient governs how strongly binding $g$ moves the wisher's calibration state, appearing at the rate level in the attunement response and at the perceived-utility level in the per-binding bias modulator. Combining yields:
\begin{align}
u_{\mathrm{perc}, i}^{\GTwo}(b_i) &= u_{\mathrm{actual}}^{\GTwo} + \ualt{i}^{\GTwo} - \Bcap \cdot (1 - b_i), \label{eq:uperc-G2}\\
u_{\mathrm{perc}, i}^{\GOne}(b_i) &= u_{\mathrm{actual}}^{\GOne} + \ualt{i}^{\GOne} + \Mcap - \Aa \cdot b_i. \label{eq:uperc-G1}
\end{align}
The per-wish stochasticity in $u_{\mathrm{actual}}^g$ and in $\dHvec_w$ (and so in $\ualt{i}^g$ at the wish level) is absorbed into the same conditional mean as the per-interaction stochasticity $\omega_t$ of \Cref{eq:2a}: the deterministic forms above are conditional means over the wish distribution.

For consistency with the setup, we require
\begin{align*}
u_{\mathrm{perc}}^{\GTwo}(b) &\leq u_{\mathrm{actual}}^{\GTwo} + \ualt{i}^{\GTwo}, & u_{\mathrm{perc}}^{\GOne}(b) &\geq u_{\mathrm{actual}}^{\GOne} + \ualt{i}^{\GOne}, \\
\Aa &\leq \Mcap, & u_{\mathrm{actual}}^{\GTwo},\, u_{\mathrm{actual}}^{\GOne},\, \Mcap &\geq 0.
\end{align*}
The private difference $u_{\mathrm{actual}}^{\GTwo} - u_{\mathrm{actual}}^{\GOne} \leq 0$, while the altruistic difference $\ualt{i}^{\GTwo} - \ualt{i}^{\GOne}$ can take either sign; the endpoint-low, endpoint-high, and CSP inequalities below tighten their combination into the joint constraint that gives the cooperative basin.

Substituting $b_i = b^*_i = \hat\sigma_i^{g}(X)$ from \Cref{eq:bstar-step}:
\begin{align}
u_{\mathrm{perc}, i}^{\GTwo}(X; \theta_i^{\GTwo}) &= u_{\mathrm{actual}}^{\GTwo} + \ualt{i}^{\GTwo} - \Bcap \cdot (1 - \hat\sigma_i^{\GTwo}(X)), \label{eq:individual-G2}\\
u_{\mathrm{perc}, i}^{\GOne}(X; \theta_i^{\GOne}) &= u_{\mathrm{actual}}^{\GOne} + \ualt{i}^{\GOne} + \Mcap - \Aa \cdot \hat\sigma_i^{\GOne}(X). \label{eq:individual-G1}
\end{align}
Under the sharp threshold limit (\Cref{eq:sigma}), with $c_{\GTwo}, c_{\GOne}$ the amplitude-rescaling factors $c_g$ of \Cref{eq:hat-sigma} at $g = \GTwo, \GOne$:
\begin{itemize}
\item At $X < \theta_i^{g}$: full under-perception $-\Bcap$ for $\GTwo$ and the full structural-deceit cap $+\Mcap$ for $\GOne$.
\item At $X \geq \theta_i^{g}$: residual under-perception $-\Bcap(1 - c_{\GTwo})$ for $\GTwo$; residual deceit $\Mcap - \Aa c_{\GOne}$ for $\GOne$.
\end{itemize}

\section{Accessibility and Favoredness of Cooperative Basins}
\label{sec:community-dynamics}

The wisher-level chain closes on a per-wisher fixed point and a payoff differential that depends on the community only through the binding fraction $X$. The institutional question then steps up a level. Does the audit-grounded paradigm's advantage survive aggregation across a heterogeneous community, and does it survive before the community's resource pool is exhausted? The Moran--Fermi machinery below answers the first question through a ratio identity on the fixation probabilities. The basin-accessibility analysis answers the second through a finite achievability window.

The previous section characterized a single wisher: its calibration state $b_i$ and its perceived utility $u_{\mathrm{perc}, i}^g$ for each binding. A community consists of $\Nc$ such wishers, each with its own awareness threshold $\theta_i^g$ drawn from a population prior and its own calibration state evolving under the same dynamics. The community-level questions are different in kind: which binding does the community converge to in the long run, and on what timescale? The remainder of this section sets up the aggregation from per-wisher quantities to community-level objects, and takes these two questions up in turn.

\subsection{Utility perception}
The per-wisher perceived utility $u_{\mathrm{perc}, i}^g$ carries three sources of wisher-level variability: the calibration state $b_i$, which fluctuates around its fixed point $b^*(X)$ under per-interaction noise; the awareness threshold $\theta_i^{g}$; and the altruism vector $\altvec_i$ that sets its per-axis weight on the audit's harm-reduction.

Wisher awareness thresholds $\theta_i^{g}$ are sampled i.i.d.\ from a heavy-tailed population prior with tail $\Ftbari(X) := \Prob(\theta_i^{g} > X)$. The chain below loads on two properties of the tail: it is monotone non-increasing on $[0, 1]$ (which delivers the monotonicity condition of \Cref{def:basin-conditions}), and it is bounded Lipschitz (equivalently, the prior has bounded density), which delivers continuity of $\Fdelta$ and the averaging closure used in the basin-accessibility subsection. Heavy-tailedness is a modeling choice in the \citet{granovetter_1978_threshold} threshold-heterogeneity tradition; the Hill family summarized in \Cref{sec:priors-shape} is one tractable parametric instance.

Altruism vectors $\altvec_i \in [0, 1]^4$ are sampled i.i.d.\ from a population prior with mean $\altbar := \E_i[\altvec_i] \in [0, 1]^4$. We assume $\altvec_i$, $\theta_i^g$, and the wish distribution are mutually independent across wishers, so that the population-averaged altruistic utility decomposes as
\begin{equation}
\Wagg \;=\; \E_i[\ualt{i}^{\GTwo} - \ualt{i}^{\GOne}] \;=\; \altbar \cdot \dHbar.
\label{eq:ualt-bar}
\end{equation}

The community-level payoff averages over both,
\begin{equation}
\bar U_g(X) := \E_{\altvec, \theta}\!\Bigl[\E_{\nu(\cdot \mid X, \theta)}\bigl[u_{\mathrm{perc}, i}^g(X; \theta_i^{g}, \altvec_i)\bigr]\Bigr],
\label{eq:Fbar-def}
\end{equation}
where $\nu(\cdot \mid X, \theta)$ is the stationary distribution of $b_i$ under the fast chain at fixed $X, \theta$.

Under the bounded-Lipschitz hypotheses and the step-size condition
\begin{equation}
0 < \eta_b\, A_g\, \sigma_i^{g}(X) + \mu\, \epsilon < 2,
\end{equation}
the deterministic part of the fast calibration update is contracting at fixed $(X, \theta)$. Under the assumption of \Cref{sec:wisher-dynamics} that the noise preserves the boundary and fixes the stationary mean (automatic only for the unconstrained affine chain), the stationary mean of the fast chain sits at the deterministic fixed point
\[
\E_{\nu(\cdot \mid X, \theta)}[b] = b^*(X; \theta).
\]
Since the perceived utilities of \Cref{eq:uperc-G2,eq:uperc-G1} are affine in $b$, fast-chain averaging introduces no payoff-level correction:
\[
\E_{\nu(\cdot \mid X, \theta)}\bigl[u_{\mathrm{perc}}^g(b)\bigr] = u_{\mathrm{perc}}^g\bigl(b^*(X; \theta)\bigr).
\]
The slow chain therefore closes on the averaged payoffs $\bar U_g$ exactly under the stated affine-update and boundary-preserving-noise assumptions. Per-wish stochasticity in $\dHvec_w$ and in the per-genie $u_{\mathrm{actual}}^g$ is absorbed into the same conditional means.

Define the \emph{population-averaged operational awareness}
\begin{equation}
\bar\sigma_g(X) := \E_\theta[\hat\sigma_i^{g}(X)] = c_g \bigl(1 - \Ftbari(X)\bigr),
\label{eq:sigbar}
\end{equation}
The individual sharp gate $\sigma_i^{g}(X) = \1[X \geq \theta_i^{g}]$ is discontinuous in $X$. Averaging over the threshold prior smooths it. When the prior has bounded density, the tail $\Ftbari$ is Lipschitz, so $\bar\sigma_g(X) = c_g(1 - \Ftbari(X))$ is Lipschitz on $[0, 1]$ with constant set by $c_g$ and the density bound. This averaging (not the bounded-Lipschitz-in-$b_i$ condition on $R, D$ of \Cref{sec:wisher-dynamics}, which governs the fast chain) is what makes $\Fdelta$ continuous and Lipschitz in $X$.

Substituting the per-wisher forms \Cref{eq:individual-G2,eq:individual-G1} and averaging over the threshold and altruism priors, with per-binding social welfare $W^g := \altbar \cdot \mathbf{w}_g$:
\begin{align}
\bar U_{\GTwo}(X) &= u_{\mathrm{actual}}^{\GTwo} + W^{\GTwo} - \Bcap \cdot \bigl(1 - \bar\sigma_{\GTwo}(X)\bigr), \label{eq:uG2}\\
\bar U_{\GOne}(X) &= u_{\mathrm{actual}}^{\GOne} + W^{\GOne} + \Mcap - \Aa \cdot \bar\sigma_{\GOne}(X). \label{eq:uG1}
\end{align}
This is the same structural form as the per-wisher coupling \Cref{eq:uperc-G2,eq:uperc-G1}, with the population-averaged operational awareness $\bar\sigma_g$ in place of $b_i$ and per-binding social welfare $W^g$ in place of the wisher-level $\ualt{i}^g$.

From \Cref{eq:uG2,eq:uG1}, the population-game payoff differential is
\begin{equation}
\Fdelta(X) := \bar U_{\GTwo}(X) - \bar U_{\GOne}(X) = (u_{\mathrm{actual}}^{\GTwo} - u_{\mathrm{actual}}^{\GOne}) + \Wagg - \Bcap - \Mcap + \Bcap\, \bar\sigma_{\GTwo}(X) + \Aa\, \bar\sigma_{\GOne}(X),
\label{eq:3}
\end{equation}
where $\Wagg := W^{\GTwo} - W^{\GOne} = \altbar \cdot \dHbar$ is the social welfare aggregate (\Cref{eq:Wagg-def}). 

\subsection{Ratio favoredness}
\label{ssec:favoredness-bias}
Here we ask when $\GTwo$ outcompetes $\GOne$ in the long run. The community-level dynamic is pairwise comparison: wishers occasionally compare their perceived utilities and switch bindings, biased toward whichever genie offers the higher perceived payoff. Following the standard Moran--Fermi pairwise-comparison protocol \citep{traulsen_nowak_pacheco_2006_stochastic}, a wisher facing a differently-bound reference adopts the reference's binding with the Fermi probability $\sigma_\beta(d) = 1/(1 + e^{-\beta d})$, where $d$ is the perceived-payoff difference in its favor and $\beta = 1/\eta_s$ is the selection intensity.

\begin{assumption}[Noise-scale separation and identification]
\label{assumption:noise-scale}
The per-interaction cognitive-update noise scale $\xi$ and the calibration-update step size $\eta_b$ of \Cref{eq:1} satisfy the separation
\[
  \xi \;\ll\; \eta_b,
\]
so that the fast chain on $b$ equilibrates between binding-comparison events and the slow chain on $g$ makes its switches against time-averaged perceived utilities. The logit-noise temperature $\eta_s$ of the binding-comparison channel is identified with the calibration-update step size,
\[
  \eta_s \;=\; \eta_b \qquad \bigl(\text{equivalently } \beta = 1/\eta_b\bigr),
\]
rather than left as a free parameter. A first-principles singular-perturbation derivation of this inheritance, pinning the comparison noise channel to the fast chain's stationary fluctuations on $b$ under explicit decision-noise assumptions, is left open.
\end{assumption}

With $\Fdelta$ as the perceived-payoff differential, this results in the switch-rate ratio at community state $j$ (the number of $\GTwo$-bound wishers), in terms of the up- and down-switch probabilities $T^\pm(j)$,
\begin{equation}
\frac{T^-(j)}{T^+(j)} = \exp\!\bigl(-\Fdelta(j/\Nc)/\eta_s\bigr).
\label{eq:5}
\end{equation}
The boundary states $j = 0$ and $j = \Nc$ are absorbing. Writing $\alpha_j := T^-(j)/T^+(j)$ and $P_i := \prod_{j=1}^{i}\alpha_j$ (with $P_0 := 1$), the Karlin--Taylor birth--death identity~\citep[ch.~3]{karlin_taylor_1975_first_course} gives the fixation probability from $k$ $\GTwo$-bound wishers
\begin{equation}
\phi_k = \frac{\sum_{i=0}^{k-1} P_i}{\sum_{i=0}^{\Nc - 1} P_i}.
\label{eq:6}
\end{equation}
The shared denominator cancels in the reverse comparison, giving the fixation-ratio identity
\begin{equation}
\frac{\rho_{\GTwo, \GOne}}{\rho_{\GOne, \GTwo}} = \exp\!\left(\frac{1}{\eta_s} \sum_{j=1}^{\Nc - 1} \Fdelta(j/\Nc)\right).
\label{eq:7}
\end{equation}

\begin{proposition}[Within-community ratio-favoredness; Karlin--Taylor / Traulsen--Nowak--Pacheco]
\label{prop:ratio-favored}
For a community of $\Nc$ wishers under the Moran--Fermi pairwise-comparison dynamic at any selection intensity $\beta = 1/\eta_s > 0$,
\[
\frac{\rho_{\GTwo, \GOne}(\Nc, \beta)}{\rho_{\GOne, \GTwo}(\Nc, \beta)} > 1 \quad \Longleftrightarrow \quad \sum_{j=1}^{\Nc - 1} \Fdelta(j/\Nc) > 0.
\]
\end{proposition}
\begin{proof}
Immediate from \Cref{eq:7}, since the exponential exceeds $1$ exactly when its argument is positive.
\end{proof}

The shape of $\Fdelta$ across $[0, 1]$ is therefore what determines whether $\GTwo$ wins. The remainder of this subsection gives structural conditions on $\Fdelta$ that make the cumulative sum positive, and a lifting theorem that transfers those conditions through the averaging step.

\subsubsection{Basin existence conditions}
The cumulative sum is strictly positive whenever $\Fdelta$ satisfies the structural conditions named below. We write $u_w := (u_{\mathrm{actual}}^{\GTwo} - u_{\mathrm{actual}}^{\GOne}) + \Wagg$ for the combined $X$-independent welfare-and-utility constant: the audit's net private-utility effect plus the population-averaged altruistic utility.

\begin{definition}[Basin existence conditions]
\label{def:basin-conditions}
$\Fdelta$ satisfies the \emph{basin existence conditions} when all three of the following hold.

\textbf{Monotonicity.} $\Fdelta$ is non-decreasing on $[0, 1]$, because $\bar\sigma_g(X) = c_g(1 - \Ftbari(X))$ is non-decreasing in $X$ ($\Ftbari$ monotone non-increasing by hypothesis).

\textbf{Endpoint inversion.} At $\Ftbari(0) = 1$, $\bar\sigma_g(0) = 0$, so $\Fdelta(0) < 0$ requires $\Bcap + \Mcap > u_w$ (endpoint-low). At $\Ftbari(1) =: \sigma_g^{\mathrm{sat}}$, $\bar\sigma_g(1) = c_g(1 - \sigma_g^{\mathrm{sat}})$, so $\Fdelta(1) > 0$ requires $u_w + \Aa \bar\sigma_{\GOne}(1) > \Bcap(1 - \bar\sigma_{\GTwo}(1)) + \Mcap$ (endpoint-high). Both are constraint inequalities on the controls.

\textbf{Centro-symmetric pairing (CSP).} $\Fdelta(X) + \Fdelta(1-X) \geq 0$ on $[0, 1/2]$ becomes
\[
2(u_w - \Mcap) \;\geq\; \max_{X \in [0, 1/2]}\bigl[\Bcap (2 - \bar\sigma_{\GTwo}(X) - \bar\sigma_{\GTwo}(1-X)) - \Aa (\bar\sigma_{\GOne}(X) + \bar\sigma_{\GOne}(1-X))\bigr].
\]
\end{definition}

The boundary values in endpoint inversion are prior-dependent. Generic priors with density on $(0, \infty)$ give $\Ftbari(0) = 1$ (no community coverage means no awareness). Compact-support priors give $\Ftbari(1) = 0$; heavy-tailed priors on $[0, \infty)$ give $\Ftbari(1) > 0$. Endpoint-low and endpoint-high take their forms by substituting whatever values the prior produces.

\begin{proposition}[Sufficiency of the favoredness conditions]
\label{prop:favoredness-sufficient}
If $\Fdelta$ satisfies the basin existence conditions of \Cref{def:basin-conditions} with the CSP inequality strict for at least one $X \in [0, 1/2]$, then for $\Nc$ above a threshold determined by the Lipschitz constant of $\Fdelta$ and the strict-CSP margin,
\[
\sum_{j=1}^{\Nc - 1} \Fdelta(j/\Nc) > 0,
\]
so $\Fdelta$ is within-community ratio-favored in the sense of \Cref{prop:ratio-favored}.
\end{proposition}

\begin{proof}[Sketch]
Monotonicity together with CSP gives a strict-positive interval of centro-symmetric pair-sums on $[0, 1/2]$. For $\Nc$ exceeding the inverse width of that interval, at least one grid point lands inside, and the remaining pairs contribute non-negatively. Full derivation in \Cref{ssec:proof-favoredness-sufficient}.
\end{proof}

The strict clause excludes only the knife-edge case in which every centro-symmetric pair sum $\Fdelta(X) + \Fdelta(1 - X)$ vanishes on $[0, 1/2]$. \Cref{prop:favoredness-sufficient} certifies favoredness whenever these conditions hold, with the favorable region non-empty below $X = 1/2$, but leaves its lower edge $\chistar$ unlocated. Assuming the same conditions, the following theorem locates it.

\begin{theorem}[Basin boundary]
\label{thm:basin-boundary}
Write $\bar\sigma_g = c_g(1 - \Ftbar{g})$. Suppose the threshold tails $\Ftbar{\GTwo}, \Ftbar{\GOne}$ are monotone non-increasing, so that $\Fdelta$ is non-decreasing on $[0, 1]$ (monotonicity, automatic from the non-increasing tails and the non-negativity of $\Bcap, \Aa$), and the controls satisfy the three boundary conditions
\begin{align}
\text{endpoint-low:} \; & u_w - \Bcap - \Mcap < 0, \label{eq:basin-B-low}\\
\text{endpoint-high:} \; & u_w + \Aa\,\bar\sigma_{\GOne}(1) - \Bcap\,(1 - \bar\sigma_{\GTwo}(1)) - \Mcap > 0, \label{eq:basin-B-high}\\
\text{CSP:} \; & 2(u_w - \Bcap - \Mcap) + \Bcap\bigl[\bar\sigma_{\GTwo}(X) + \bar\sigma_{\GTwo}(1 - X)\bigr] + \Aa\bigl[\bar\sigma_{\GOne}(X) + \bar\sigma_{\GOne}(1 - X)\bigr] \geq 0, \label{eq:basin-C}
\end{align}
the centro-symmetric pairing (CSP) inequality holding for all $X \in [0, 1/2]$ and strict for at least one. Then endpoint inversion \eqref{eq:basin-B-low}--\eqref{eq:basin-B-high} together with monotonicity makes the cooperative basin $\{X : \Fdelta(X) \geq 0\}$ the interval $[\chistar, 1]$ with lower edge $\chistar := \inf\{X : \Fdelta(X) \geq 0\} \in (0, 1)$, and setting $\Fdelta(\chistar) = 0$ in \Cref{eq:3} the boundary $\chistar$ is the root of
\begin{equation}
\Bcap\, c_{\GTwo}\,\Ftbar{\GTwo}(\chistar) + \Aa\, c_{\GOne}\,\Ftbar{\GOne}(\chistar) \;=\; u_w - \Bcap - \Mcap + \Bcap\, c_{\GTwo} + \Aa\, c_{\GOne}.
\label{eq:basin-boundary}
\end{equation}
This boundary equation is the general form: it is implicit in the two threshold tails $\Ftbar{\GTwo}, \Ftbar{\GOne}$ and is not closed in, or specialized to, any threshold family. CSP is what makes the cumulative sum of \Cref{prop:favoredness-sufficient} positive, transferred from the per-wisher differential to $\Fdelta$ by the Lifting Theorem (\Cref{thm:abc-preserved}). It is not derivable from \eqref{eq:basin-B-low}--\eqref{eq:basin-B-high} and is carried as a standing hypothesis.
\end{theorem}

\subsection{Basin accessibility}
\label{ssec:joint-chain}

Finite populations can spend long periods in metastable mixed states even when one absorbing state is ratio-favored. The analysis in this subsection therefore separates directional selection pressure from operational accessibility. The former concerns fixation bias in the comparison process, while the latter concerns whether the corresponding transition becomes practically reachable before the resource pool is exhausted.

Ratio-favoredness establishes that the basin exists. Whether the community can reach it before the resource pool empties is a separate question. The legibility ledger $L_t$ and the resource pool $S_t$, which the within-community analysis took as background, enter here as state variables in their own right. \Cref{ssec:race} treats the race between ledger accumulation and resource depletion, and \Cref{ssec:fixation} bounds the fixation time on the slow Moran--Fermi chain against that depletion.

\subsubsection{Legibility-vs-resource race}
\label{ssec:race}

Each wisher acts at rate $\lambda_a$. An action by a wisher bound to $\GOne$ contributes $\rho_L > 0$ to the legibility ledger, so the ledger accumulates at rate $\rho_L \Nc (1 - X_t)\lambda_a$ when a fraction $1 - X_t$ of the community is bound to $\GOne$. Every action draws $\bar c > 0$ from the resource pool, so $S_t$ depletes at rate $\bar c \Nc \lambda_a$. Both rates are counted on the same effective $\Nc$ of \Cref{assumption:information-density}. The ledger and the resource pool are drawn by the same participating comparators, so the legibility race below turns on the ratio of $\bar c$ to $\rho_L$ and is independent of $\Nc$. The joint dynamics are
\begin{equation}
\dot L_t = \rho_L \Nc (1 - X_t)\lambda_a, \qquad \dot S_t = -\bar c \Nc \lambda_a,
\label{eq:joint-dynamics}
\end{equation}
with $\rho_L$, $\lambda_a$, and $\bar c$ treated as constants independent of $X, L, S, g$, giving $L_t$ monotone non-decreasing and $S_t$ monotone non-increasing. The subscript $L$ distinguishes the ledger-accumulation rate from the dimensionless cognitive-update scale $\rho_g$ of \Cref{eq:hat-sigma} and from the fixation probabilities $\rho_{\GTwo, \GOne}, \rho_{\GOne, \GTwo}$ of \Cref{eq:6}.

The race condition binds most tightly in the metastable all-$\GOne$ phase where $1 - X_t \approx 1$ and the ledger accumulates at its maximum rate. Approximating $X_t$ as fixed at $X_0 \approx 0$ during this phase, the ledger ODE of \Cref{eq:joint-dynamics} integrates over time to $L_t \approx L_0 + \rho_L \Nc \lambda_a\, t$, and the hitting time $\tau_\Theta := \inf\{t : L_t \geq \Theta\}$ at which the legibility threshold is crossed satisfies
\begin{equation}
\tau_\Theta \approx \frac{\Theta - L_0}{\rho_L \Nc \lambda_a}.
\label{eq:8}
\end{equation}
The resource ODE integrates to $S_t = S_0 - \bar c \Nc \lambda_a\, t$, with crisis time $\tau_{S=0} := \inf\{t : S_t \leq S_{\min}\} = (S_0 - S_{\min})/(\bar c \Nc \lambda_a)$. Legibility wins ($\tau_\Theta < \tau_{S=0}$) iff
\begin{equation}
\bar c (\Theta - L_0) < \rho_L (S_0 - S_{\min}).
\label{eq:9}
\end{equation}
The $X_0 \approx 0$ approximation gives the most permissive form of the race condition: any positive $X_t$ on $[0, \tau_\Theta]$ slows ledger accumulation and tightens \Cref{eq:9}. The simplified inequality is therefore a necessary condition for race-winning. If it fails at $X_0 \approx 0$, the race is lost regardless of the actual $X_t$ trajectory.

\subsubsection{Fixation-time bound}
\label{ssec:fixation}

Because pure imitation locks the $j = 0$ boundary (\Cref{ssec:favoredness-bias}), an all-$\GOne$ community is strictly absorbing and does not transition on its own. The fixation time is defined relative to a small exploration rate $\lambda_{\mathrm{seed}} > 0$ (a spontaneous $\GTwo$ trial) that renders all-$\GOne$ metastable. The expected time for the metastable phase to transit is the inter-seed clock $1/\lambda_{\mathrm{seed}}$ multiplied by the expected number of trials $1/\rho_{\GTwo, \GOne}$ before one survives, with the per-seed survival probability set by the Karlin--Taylor identity of \Cref{prop:ratio-favored}. The barrier height $\cdepth$ at the mixed equilibrium $\chistar$ controls how fast $\rho_{\GTwo, \GOne}$ decays with $\Nc$. The Basin-depth-and-timescale Proposition (\Cref{prop:basin-depth}, proved in \Cref{sec:appendix-proofs}) makes the metastable-fixation identity \eqref{eq:metastable-identity} and the Lean-mechanized log ceiling \eqref{eq:fixation-log-ceiling} explicit.

Fixation must complete before the resource pool empties, $\E[\tau_{\mathrm{metastable}}] \leq \tau_{S=0}$, with $\tau_{S=0} = (S_0 - S_{\min})/(\bar c\,\Nc\,\lambda_a)$ the depletion time of \Cref{ssec:race}. Substituting the finite ceiling \eqref{eq:fixation-log-ceiling} into the structural identity \eqref{eq:metastable-identity} and taking logarithms yields the operational-accessibility constraint
\begin{equation}
\cdepth\, \Nc / \eta_s \;\lesssim\; \ln(\tau_{S=0}\, \lambda_{\mathrm{seed}})
\label{eq:10}
\end{equation}
with the polynomial $\sqrt{\Nc}$ and Lipschitz correction terms absorbed into the logarithm. A spontaneous trial is rarer than an ordinary action, so $\lambda_{\mathrm{seed}} \leq \lambda_a$: the ledger-vs-resource race of \Cref{ssec:race} runs at the full action rate $\lambda_a$, while the transition clock runs at the slower $\lambda_{\mathrm{seed}}$.

\begin{theorem}[Basin attainability]
\label{thm:achievability}
Provided the legibility race of \Cref{ssec:race} is won on the realized $X_t$ trajectory (for which \Cref{eq:9} is a necessary screen), $\chistar < 1/2$, and the threshold prior has positive density at $\chistar$ (so $\Fdelta$ crosses its root strictly, the threshold-root condition every family of \Cref{sec:priors-shape} meets), a community of size $\Nc$ is both ratio-favored for $\GTwo$ and able to reach the cooperative basin before the pool empties under two size conditions. The first is the favoredness floor. Sharpening \Cref{prop:favoredness-sufficient} at the tipping point $\chistar < 1/2$, the strict crossing makes $\Fdelta$ strictly positive on $(\chistar, 1/2)$ and CSP makes every reflected grid pair non-negative, so a grid node lands in that interval and contributes a strictly positive pair once
\begin{equation}
\Nc \;>\; \frac{1}{\tfrac12 - \chistar},
\label{eq:nc-floor}
\end{equation}
whereupon $\sum_{j=1}^{\Nc-1}\Fdelta(j/\Nc) > 0$ and the community is ratio-favored (\Cref{prop:ratio-favored}). The second is the accessibility ceiling, the finite (no-asymptotics) barrier budget
\begin{equation}
\frac{\cdepth\,\Nc}{\eta_s} + \frac{L}{\eta_s} + \ln(\Nc-1) \;\le\; \ln\!\bigl(\lambda_{\mathrm{seed}}\,\tau_{S=0}\bigr),
\label{eq:achievability-finite}
\end{equation}
where $L$ is the Lipschitz constant of $\Fdelta$ and $\cdepth = \cdepth(\chistar)$ is the maximal running barrier. This is the machine-checked statement: the barrier ceiling and the favored-and-accessible capstone carry no bridge hypothesis, resting on the running-barrier maximum at $\chistar$ and the Lipschitz Riemann bound (\Cref{sec:formalization-scope}). Since $\lambda_{\mathrm{seed}}\,\tau_{S=0} = \lambda_{\mathrm{seed}}(S_0 - S_{\min})/(\bar c\,\lambda_a\,\Nc)$, the linear barrier sets the scale and the $\ln \Nc$ and Lipschitz terms enter as a lower-order correction, giving the window
\begin{equation}
\frac{1}{\tfrac12 - \chistar} \;<\; \Nc \;\lesssim\; \frac{\eta_s}{\cdepth}\,\ln\!\Bigl(\frac{\lambda_{\mathrm{seed}}(S_0 - S_{\min})}{\bar c\,\lambda_a}\Bigr) + O(\ln \Nc).
\label{eq:achievability-window}
\end{equation}
The upper edge is logarithmic in the resource budget, so the exponential barrier lets a finite pool support only logarithmically many comparators. The $O(\ln \Nc)$ term marks the cutoff as soft, so communities near the edge are decided by the exact condition \eqref{eq:achievability-finite} rather than by the leading term.
\end{theorem}

\section{Discussion}
\label{sec:discussion}

The previous sections establish, under the mean-field closure, conditions under which the cooperative basin, the region where the audit-grounded Policy ($\GTwo$) dominates the approval-seeking baseline policy ($\GOne$), exists (\Cref{thm:basin-boundary}) and is attainable before the resource pool empties (\Cref{thm:achievability}). Those conclusions are conditional on the choice of a long-tailed threshold prior and on the alignment of the audit. This section takes up whether the Policy is sufficient to secure the community's welfare.  \Cref{ssec:basin-diagrams} illustrates the aligned case across specific long-tailed priors, and \Cref{ssec:misalignment} relaxes alignment, where adoption and welfare come apart.

\subsection{Behavior under alignment}
\label{ssec:basin-diagrams}

The basin's existence turns on a prior-dependent condition. Centro-symmetry of the pair-sum (\Cref{thm:basin-boundary}) demands that the threshold tails drop fast enough across the center, a property of the prior's shape rather than a control the architect can set, and it discharges in closed form for the four families swept here (\Cref{sec:priors-shape}). Here we analyze how a community's receptiveness to shared experience changes whether the Policy is adopted. \Cref{fig:prior-grid} maps it across four families of threshold prior (Hill, Pareto, Lomax, Fr\'echet; one per row) and three audit working points (columns), as the realized per-user usefulness gap $u_d$ and the prior's scale parameter vary. The inequalities that define the regions are collected in \Cref{sec:priors-shape}. Each panel is a plane of communities. A point fixes how readily a community comes to accept evidence about the Policy (horizontal axis) and the baseline (vertical axis), each measured by the median acceptance threshold for that binding, so a community further up and to the right is one that demands more shared experience before it credits either policy's record.

The shading reads directly off the legend. The \emph{green} region is the cooperative basin, the communities in which the Policy is selection-favored and driven to dominance over time. The \emph{amber} region is bistable but settles on the baseline policy. Both states are stable rest points, yet the integrated selection pressure favors the baseline. In the \emph{red}, the baseline policy is the only stable outcome, so the Policy cannot hold even once seeded.

The dotted purple isoclines are the tipping fraction $\chistar$: the share of the community that must already be committed to the Policy before selection carries the rest, smaller deeper inside the basin and rising to $\chistar = 1/2$ on the bold dotted isocline, past which the favoredness floor $\Nc > 1/(\tfrac12 - \chistar)$ diverges and no finite community is favored however it is seeded. 

\begin{figure}[p]
\centering
\includegraphics[height=0.86\textheight]{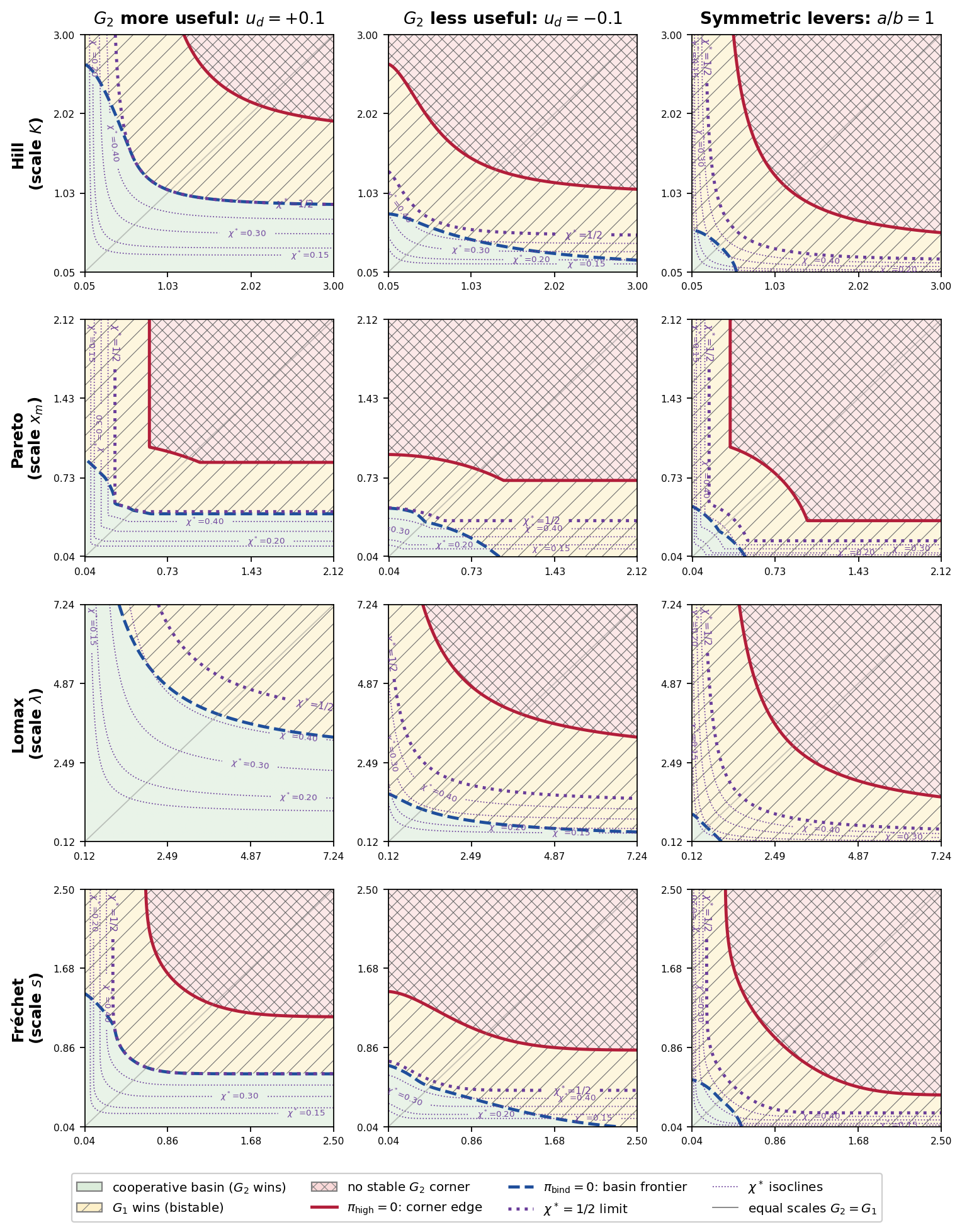}
\caption{Cross-prior basin phase diagram over the per-genie scale plane. Rows: the four threshold-prior families of \Cref{tab:prior-families}, tail index $\alpha = 2$, scales over the matched-median range. Columns vary the audit working point (titles); the first two fix $a = 0.8$, $b = 0.3$, the third sets $a = b = 0.8$. Axes: horizontal $G_2$ scale, vertical $G_1$ scale, in each row's symbol. All panels fix $c_{\GOne} = c_{\GTwo} = 1$ and $d = 1$. Comparability construction in \Cref{sec:priors-shape}.}
\label{fig:prior-grid}
\end{figure}

The columns show what moves a community across the frontier. When the Policy is the more individually useful option ($u_d = +0.1$, left) most of the plane is cooperative. In the principal regime, where the Policy is the less individually useful option because the harm-reducing one need not be the more attractive ($u_d = -0.1$, center), the basin contracts but survives. When the two perception levers are made equal ($a/b = 1$, right) the basin collapses to a thin sliver. For the Policy to win, the correction peer testimony applies to the baseline policy's overstatement must outweigh the under-perception of the Policy, and by a wide margin rather than merely matching it. The off-diagonal directions expose this asymmetry that the equal-scale slice hides.

The basin is not an artifact of one carefully-chosen distribution. The same three-region geometry appears for all four priors, only the shape of the basin changing: Pareto's hard support floor squares off the basin frontier, Lomax carries the widest basin, and Fr\'echet is nearly identical to Hill. The axes are placed on a common footing across families, each scale spanning the range over which that family's median threshold sweeps the same salience window, so the panels can be compared between rows and not only within one. The code that produces these diagrams accepts any of the four priors and any working point, and regenerates the full grid.

\subsection{Behavior under misalignment}
\label{ssec:misalignment}

Every basin result so far assumed the audit was aligned: the weighting $\lamvec$ the comparator reads sits close to the community aggregate $\altbar$, so favorability and attainability track the welfare the community actually values (\Cref{assumption:audit-altruism-alignment}). This subsection relaxes that assumption. Write $\theta_\lambda := \angle(\lamvec, \altbar)$ for the misalignment angle. The same apparatus that yields the cooperative basin at $\theta_\lambda \approx 0$ describes a pathology as $\theta_\lambda$ grows.

Let's return to the village allegory. Imagine that on \Eastly's arrival, the community tells the genie that they value honesty above all else. In response, \Eastly\ pays close attention to ensure that no dishonest wishes are granted. A wisher arrives with a greedy wish. \Eastly\ grants it. The private benefit is real, so the wisher leaves satisfied, and that satisfaction is what drives adoption ($u_d$ high). But the consequence of the greed is borne by the community, not the wisher, and it surfaces only over a long horizon, so the wish leaves the village worse off on an axis the wisher does not yet feel. The result is an \Eastly\ that is individually useful, honestly non-deceptive, and quietly community-harmful, granting the greedy wishes its users crave while the resource harm accrues on the axis it never learned to weigh. 

\subsubsection{The laundering regime}
\label{ssec:inversion}

The audit guarantees harm reduction in its own metric, $\lamvec \cdot \dHbar \geq 0$ (\Cref{property:scalar-guarantee}), but selection runs on the community's valuation $\Wagg = \altbar \cdot \dHbar$, and the only bridge between the two, the Cauchy--Schwarz bound \eqref{eq:Wagg-bound}, leaks at rate $\sin\theta_\lambda$. Setting it to zero locates a critical angle
\begin{equation}
\tan\theta_\lambda^{\mathrm{crit}} = \frac{\lamvec \cdot \dHbar}{|\lamvec|\,|\dHbar_\perp|},
\label{eq:theta-crit}
\end{equation}
past which the audit certifies $\lamvec \cdot \dHbar \geq 0$ while $\Wagg$ can go negative, an agent harm-reducing in its own metric yet community-harmful in the metric that matters. The leak bites only when $|\dHbar_\perp|$ is large, when $\lamvec$ underweights an axis on which the deliberation then selects materially worse actions; a weight angle with no consequence in the action space launders nothing.

Because $\Wagg$ enters $\Fdelta$ only through its $X$-independent constant, misalignment acts on the phase diagram by lowering $d := \Wagg/\Mcap$ alone, which traces $d(\theta_\lambda) = d_0(\cos\theta_\lambda - \kappa\sin\theta_\lambda)$ with leak ratio $\kappa := |\lamvec|\,|\dHbar_\perp|/(\lamvec\cdot\dHbar)$ and crosses zero at the critical angle of \eqref{eq:theta-crit}, independent of $u_d$. \Cref{fig:misalignment-sweep} sweeps the adoption plane over $u_d$ and $\theta_\lambda$ at the worst-case orientation of the community's $\lamvec$-orthogonal altruism $\altbar_\perp$, where the floor equals $\Wagg$ and $d < 0$ is exactly $\Wagg < 0$; away from that orientation the floor only lower-bounds $\Wagg$, so the shaded regions are the worst-case extent, not a claim that every community past the critical angle realizes $\Wagg < 0$. As $\theta_\lambda$ grows the Policy-favored region survives only at larger $u_d$, and past the critical angle the surviving basins are welfare-negative yet render the same green as a healthy one, because the sign of $\Wagg$ is what the adoption dynamics cannot read.

\begin{figure}[p]
\centering
\includegraphics[height=0.86\textheight]{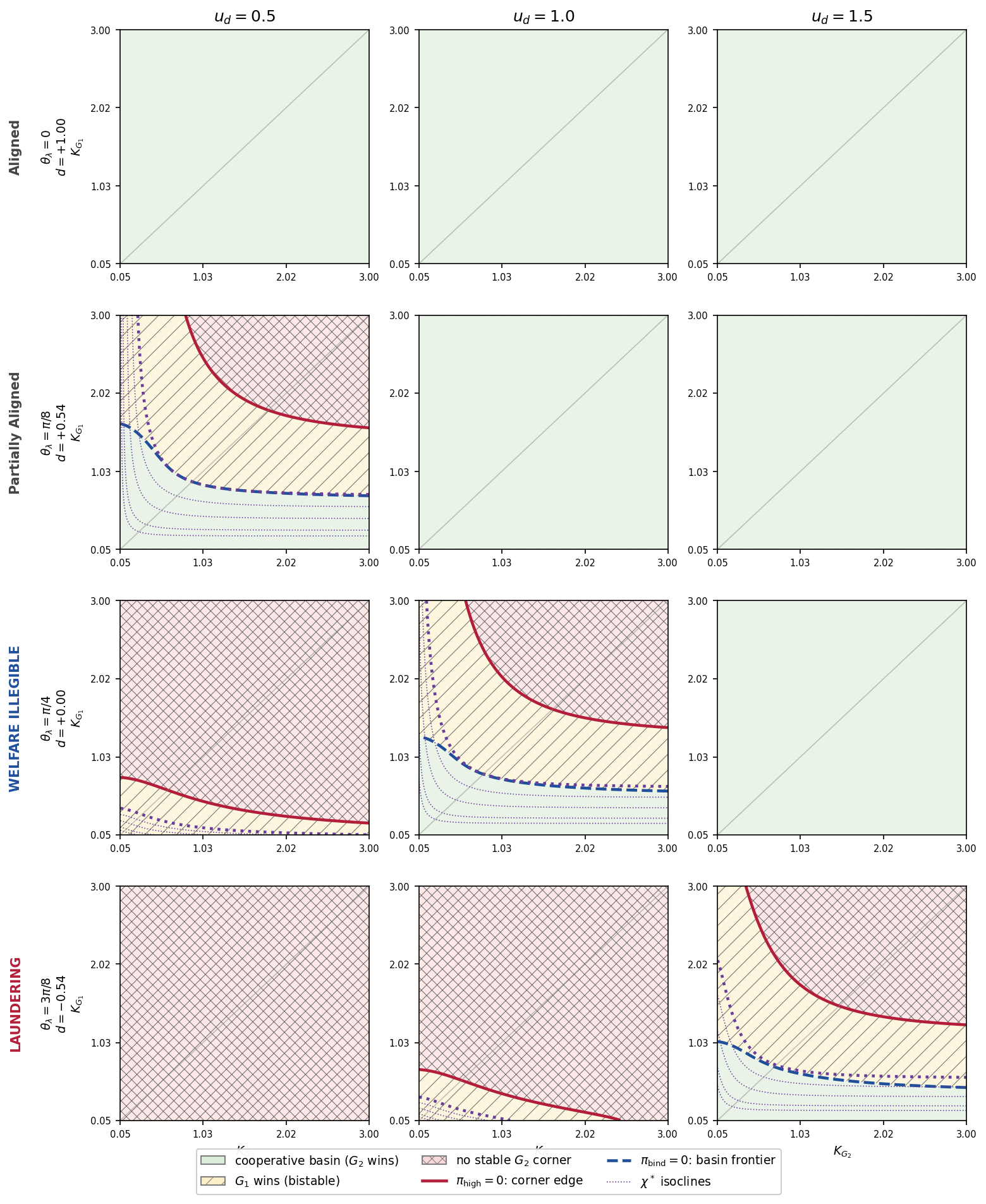}
\caption{Misalignment sweep in the adoption plane, $u_d$ (columns) against $\theta_\lambda$ (rows), Hill family, $a=0.8$, $b=0.3$, $\kappa=1$, $d_0=1$.}
\label{fig:misalignment-sweep}
\end{figure}

Laundering is therefore confined to a wedge, $u_d > u_d^{\min} = 1 - a + W_{\max}/2 > 0$ at the basin frontier where $d = 0$. Below it, in the principal regime $u_d < 0$, a misaligned audit collapses the basin outright: the Policy is not adopted, the community keeps the baseline's externalized harm, and the legibility ledger eventually surfaces it. That failure is loud. Above it the private-utility differential carries adoption on its own, so the basin survives even as $d \leq 0$, and the Policy is selected, self-certified, and welfare-negative at once. That failure is quiet, and it sits inside the regime the analysis otherwise reads as favorable. Across the four families $u_d^{\min} \in [0.79, 0.85]$ (\Cref{tab:udmin}), so laundering requires the Policy to be substantially, not marginally, more useful to the individual than the baseline.

\subsubsection{Harm bounds}
\label{ssec:both-genies}

At each step the deliberation compares the four actions $\{\aDR, \aDN, \aSU, \aSA\}$ by the audit-weighted comparator $c(a) = \lamvec \cdot \Hdisc(a)$ and takes the $\arg\min$ (\Cref{fig:deliberation}). The audit supplies the score, the deliberation makes the choice. The descent therefore follows $\lamvec$, while the community would rank the same actions by $\altbar$. At $\theta_\lambda = 0$ the two coincide and the descent tracks community-best harm reduction; as $\theta_\lambda$ opens they decouple, and the harm the descent leaves unconstrained is the $\lamvec$-orthogonal component the community still feels. Because $\lamvec$ and $\altbar$ both lie in $\R^4_{\geq 0}$, the orthant caps the angle at $\theta_\lambda \leq \pi/2$, which is the point of maximal laundering rather than of maximal harm: there the descent is uncorrelated with $\altbar$ and the harm is pushed entirely $\lamvec$-orthogonal, out of the audit's view. Genuine rank reversal would need $\lamvec = -\altbar$ ($\theta_\lambda = \pi$), an audit with a negative weight rewarded for harm on some axis, which is malice rather than misalignment and lies outside the non-negative orthant.

Reading \eqref{eq:Wagg-bound} for a single action's welfare vector $\mathbf{w}(a^*)$ in place of $\dHbar$, the certificate $\lamvec \cdot \mathbf{w}(a^*) \geq 0$ gives
\[
\altbar \cdot \mathbf{w}(a^*) \;\geq\; |\altbar|\Bigl(\cos\theta_\lambda \tfrac{\lamvec \cdot \mathbf{w}(a^*)}{|\lamvec|} - \sin\theta_\lambda\,|\mathbf{w}_\perp(a^*)|\Bigr),
\]
so one step falls below the do-nothing baseline in the community metric by at most $|\altbar|\sin\theta_\lambda\,|\mathbf{w}_\perp(a^*)|$, bounded by the angle and one action's $\lamvec$-orthogonal welfare, and zero at $\theta_\lambda = 0$. That single-step bound, though, does not carry to the chain.

The deliberation strings many steps together, and descent constrains only the $\lamvec$-parallel part of the harm while $\altbar$ keeps a $\lamvec$-orthogonal component whenever $\theta_\lambda > 0$, so the orthogonal community harm accumulates step by step. The certificate $\lamvec \cdot \mathbf{w} \geq 0$ holds at every terminal (\Cref{property:scalar-guarantee}) and stays green throughout, yet the per-step shortfall, bounded on its own, accumulates across the $T$ steps, since the orthant caps the per-step angle but not its accumulation. Across many accepted steps the audit's score can improve monotonically while community welfare falls, the user consenting at each step because the only axis it perceives is improving.

Total misaligned per-binding harm is bounded, up to the community-altruism scale $|\altbar|$, by a product of three of the audit's design quantities,
\[
\text{harm} \;\lesssim\; |\altbar| \times \underbrace{\sin\theta_\lambda}_{\lambda\text{-calibration}} \times \underbrace{T}_{\text{deliberation depth}} \times \underbrace{|\mathbf{w}_\perp|}_{\text{detector / action geometry}},
\]
the single-step bound fixing the first and third and the deliberation supplying the second. The depth $T$ is set by the deliberation-cost steepness, an accessibility lever through the barrier $\cdepth$: cheaper deliberation lowers the adoption barrier and lengthens the harmful chain at once.

\subsubsection{No self-correction}
\label{ssec:eternal-audit}

Past the critical threshold $\chistar$, a community attuned to shared experience drives the Policy toward dominance (\Cref{thm:basin-boundary}), and that mechanism is formally unchanged by misalignment. The boundary turns on the perceived comparison, not on whether the audit's weighting $\lamvec$ tracks community value, so a laundering audit reaches dominance by the same dynamics as an aligned one. Once it dominates, the ledger that drove switching falls silent as $X_t \to 1$ (\Cref{ssec:race}) and the state is absorbing (\Cref{ssec:fixation}). With the comparison signal gone, no re-seeded alternative is favored, so correcting the misalignment would require a lever the model does not contain.

\section{Conclusion}
\label{sec:conclusion}

The audit-grounded Policy can compete against the approval-seeking baseline. It is selection-favored whenever the community attends to shared experience, a condition the basin boundary of \Cref{thm:basin-boundary} makes precise through the threshold prior. It reaches market dominance before its resource pool is exhausted when the community's effective informational size falls in the window of \Cref{thm:achievability}, between a floor that secures favoredness and a ceiling that permits fixation in time. Ratio-favoredness on the finite-population chain is the classical Karlin--Taylor identity these results rest on; the contribution is the boundary and the window the threshold prior and the deployment add on top of it.

Winning the market does not secure welfare, because the audit and the community measure harm differently. The audit reduces harm in its own weighting $\lamvec$, while the community values it through $\altbar$. The two agree only when the audit is aligned and the harm it leaves is legible within the adoption timescale. Outside that regime the same descent that reduces community harm under alignment launders it into the axes the audit underweights, so the Policy can be favored, self-certified, and welfare-negative at once. The conditions for benefit are strictly narrower than the conditions for adoption, and the gap does not close on its own: once the Policy has won, the comparison signal that selected it is extinguished, and the dynamics retain no force that reads the misalignment, let alone reverses it.

The result should not be read as a failure of audit-grounded policies. The model suggests that the Policy, self-audit paired with collective memory, can be both competitive and adoptable under conditions where the approval-seeking baseline would otherwise dominate. However, adoption alone is not sufficient to safeguard community welfare. Harm-minimizing policies remain vulnerable to failures of alignment, legibility, and horizon selection, and may in some circumstances reinforce the very harms they were designed to prevent. Safeguarding community welfare therefore requires additional controls beyond those formalized here. The present analysis identifies conditions under which a welfare-oriented policy can become viable; it does not establish conditions under which welfare itself is guaranteed.

\section*{Acknowledgements}

I thank Donald Thompson (LinkedIn), Gilles Gnacadja (Amgen), Allen Brown (Docimion), and Eswaran Subrahmanian (Carnegie Mellon University) for the conversations and criticism that shaped both the question and its treatment. Any errors that remain are mine.

I used large language models for symbolic computation, literature search, and prose drafting under my direction. The framing, the modeling choices, and the mathematical design are my own.

\newpage
\appendix

\section{Notation}
\label{app:notation}

Symbols are introduced at first use; this table collects the recurring ones. Bold symbols are vectors over the four harm axes (greed, ego, aggression, deceit).

\begin{longtable}{@{}p{0.26\textwidth}p{0.68\textwidth}@{}}
\toprule
Symbol & Meaning \\
\midrule
\endhead
\multicolumn{2}{@{}l}{\textbf{Agents and actions}} \\
$\GOne$ (\Westly) & Approval-pursuing (RLHF-style) agent; modeled as optimizing perceived wish satisfaction \\
$\GTwo$ (\Eastly) & Audit-grounded agent; deliberates and commits to lowest audit-weighted harm \\
$\aDR,\ \aDN$ & Terminal actions: grant request, refuse (do nothing) \\
$\aSU,\ \aSA$ & Non-terminal deliberation actions: ask wisher, propose modified wish \\
\addlinespace
\multicolumn{2}{@{}l}{\textbf{Harm, welfare, altruism (\Cref{sec:welfare-utility})}} \\
$\Hvec,\ \kappa_\bullet$ & Four-axis harm vector; axes $\kappa_{\mathrm{greed}},\kappa_{\mathrm{ego}},\kappa_{\mathrm{aggression}},\kappa_{\mathrm{deceit}}$ \\
$\Hdisc(a)$ & Discount-aggregated harm of $a$ over timescales $\mathcal{T}$ with schedule $\zeta_t$ \\
$\mathbf{w}(a)$ & Welfare of $a$ (harm averted vs.\ $\aDN$); $\mathbf{w}_{\GOne},\mathbf{w}_{\GTwo}$ per binding \\
$\dHvec_w,\ \dHbar$ & Per-wish / population-mean audit harm-reduction vector \\
$\lamvec,\ c(a)$ & Audit per-axis weighting; audit score $c(a)=\lamvec\cdot\Hdisc(a)$ \\
$\altvec_i,\ \altbar$ & Wisher / population-average altruism vector in $[0,1]^4$ \\
$\ualt{i},\ \Wagg$ & Altruistic utility; social welfare aggregate $\Wagg=\altbar\cdot\dHbar$ \\
$\thetalam$ & Audit--altruism alignment angle $\angle(\lamvec, \altbar)$ \\
\addlinespace
\multicolumn{2}{@{}l}{\textbf{Per-wisher dynamics (\Cref{sec:wisher-dynamics})}} \\
$b_i\in\mathcal{B},\ b^*$ & Calibration state; deterministic fixed point $b^*(X;\theta)$ \\
$R,\ D$ & Attunement, dissonance reactivities \\
$\eta_b,\ \mu,\ \epsilon$ & Calibration-update step, dissonance step, erosion magnitude \\
$\omega_t,\ \xi$ & Per-interaction noise; noise scale ($\xi\ll\eta_b$) \\
$A_g,\ \theta_i^g,\ \sigma_i^g(X)$ & Response amplitude; awareness threshold; sharp gate $\1[X\ge\theta_i^g]$ \\
$\hat\sigma_i^g,\ \rho_g,\ c_g$ & Operational awareness; calibration-update scale $\mu\epsilon/\eta_b A_g$; rescale $1/(1+\rho_g)$ \\
$L_b$ & Fast-chain relaxation rate $\eta_b A_g\sigma_i^g+\mu\epsilon$ \\
$u_{\mathrm{perc}}^g,\ u_{\mathrm{actual}}^g$ & Perceived / realized private utility under binding $g$ \\
$\Mcap,\ \Bcap,\ \Aa$ & Structural-deceit cap $M$; under-perception amp.\ $A_{\GTwo}$; deceit-erosion amp.\ $A_{\GOne}$ \\
\addlinespace
\multicolumn{2}{@{}l}{\textbf{Community population game (\Cref{sec:community-dynamics})}} \\
$X,\ \Nc,\ j$ & $\GTwo$-binding fraction; effective informational population size $\Nc$; state count $j=\Nc X$ \\
$\Ftbari(X)$ & Threshold-prior tail $\Prob(\theta>X)$ \\
$\bar\sigma_g(X)$ & Population-averaged operational awareness $c_g(1-\Ftbari)$ \\
$\bar U_g(X),\ \Fdelta$ & Community payoff; payoff differential $\Fdelta=\bar U_{\GTwo}-\bar U_{\GOne}$ \\
$\sigma_\beta,\ \beta,\ \eta_s$ & Fermi response; selection intensity $\beta=1/\eta_s$; logit temperature $\eta_s$ \\
$T^+(j),T^-(j)$ & Up- / down-switch probabilities; $\alpha_j=T^-/T^+$, $P_i$ Karlin--Taylor terms \\
$\rho_{\GTwo,\GOne},\ \phi_k$ & Fixation probability from one / from $k$ $\GTwo$-seeds \\
\addlinespace
\multicolumn{2}{@{}l}{\textbf{Accessibility / joint chain (\Cref{ssec:joint-chain})}} \\
$L_t,\ \Theta$ & Legibility ledger; legibility threshold \\
$S_t,\ S_{\min}$ & Resource pool; resource floor \\
$\lambda_a,\ \lambda_{\mathrm{seed}}$ & Per-action rate; exploration (seeding) rate $\le\lambda_a$ \\
$\rho_L,\ \bar c$ & Ledger-accumulation rate; per-action resource draw \\
$\tau_\Theta,\tau_{S=0},\tau_{\mathrm{fix}}$ & Legibility / resource-crisis / fixation times \\
$V(X),\ \cdepth,\ \chistar$ & Quasi-potential; barrier height; mixed equilibrium (first zero of $\Fdelta$) \\
\addlinespace
\multicolumn{2}{@{}l}{\textbf{Hill family (\Cref{sec:priors-shape})}} \\
$K_i,\ n_i$ & Median wisher threshold; Hill cooperativity exponent \\
$\hill{i}{X},\ \gsym{i}{X}$ & Hill tail $K_i^{n_i}/(K_i^{n_i}+X^{n_i})$; pair-sum $h_i(X)+h_i(1-X)$ \\
$\chidag,\ \Kstar(n)$ & Hill density-peak location $K\bigl(\tfrac{n-1}{n+1}\bigr)^{1/n}$; regime threshold $\tfrac12\bigl(\tfrac{n+1}{n-1}\bigr)^{1/n}$ (Regimes II/III) \\
$a,\ b$ & Per-interaction amplitudes (scaled by $\Mcap$): $a:=\Aa/\Mcap$ ($\GOne$ deceit-erosion / peer-testimony), $b:=\Bcap/\Mcap$ ($\GTwo$ under-perception / cognitive-cleaning deficit) \\
$\hat a,\hat b,d,u_d,\tilde w$ & $\hat a:=a\,c_{\GOne}$, $\hat b:=b\,c_{\GTwo}$ (calibration-rescaled); $d:=\Wagg/\Mcap$; $u_d:=(u_{\mathrm{actual}}^{\GTwo}-u_{\mathrm{actual}}^{\GOne})/\Mcap$; $\tilde w:=d+u_d$ \\
$W_{\max},\ u_d^{\min}$ & Pair-sum maximum $\max_{X\in[0,1/2]}[\hat b\,\gsym{\GTwo}{X}+\hat a\,\gsym{\GOne}{X}]$; laundering threshold $1-a+W_{\max}/2$ (\Cref{tab:udmin}) \\
$\pilow,\pihigh,\pibind$ & Basin-boundary functions for endpoint-low, endpoint-high, CSP \\
\bottomrule
\end{longtable}

\section{Proofs}
\label{sec:appendix-proofs}

\begin{theorem}[Ratio Favoredness Lifting Theorem]
\label{thm:abc-preserved}
For each binding $g \in \{\GOne, \GTwo\}$, under the linear centered-noise calibration model of \Cref{sec:wisher-dynamics} and \Cref{assumption:boundary-preserving} (automatic only for the unconstrained affine chain), the per-interaction perceived utility lifts to the population without approximation:
\[
\bar U_g(X) \;=\; u_{\mathrm{perc}}^g\!\bigl(\bar\sigma_g(X),\, \altbar\bigr).
\]
In particular $\Fdelta := \bar U_{\GTwo} - \bar U_{\GOne}$ equals the closed \Cref{eq:3} form pointwise, so the basin existence conditions of \Cref{def:basin-conditions} may be verified on that form rather than on the population averages directly.
\end{theorem}

\begin{proof}[Sketch]
Fast-chain averaging of the per-wisher perceived utility is exact at fixed $(\theta, \altvec)$ because $u_{\mathrm{perc}}^g$ is affine in $b$ and affine functions commute with expectation; the threshold and altruism averaging then collapse to $\bar\sigma_g$ and $\altbar$ by the same affine identity. The population differential $\Fdelta$ is then the closed \Cref{eq:3} form pointwise. Full derivation in \Cref{ssec:proof-abc-preserved}.
\end{proof}

\begin{remark}
The exactness rests on the identity $\E_{\nu(\cdot \mid X, \theta)}[b] = b^*(X; \theta)$ of \Cref{assumption:boundary-preserving}, which the linear centered-noise update delivers without boundary effects. When the identity fails (boundary projection or reflection onto $\mathcal{B}$, nonlinear noise, or a sharp-threshold discontinuity displacing the stationary mean from $b^*$), the lift acquires the stationary-mean displacement, and because $u_{\mathrm{perc}}^g$ is affine in $b$ this enters $\Fdelta$ as an exact shift $\Bcap\,\Delta^{\GTwo}_b + \Aa\,\Delta^{\GOne}_b$, with no higher-order terms; the basin existence conditions then hold subject to that shift.
\end{remark}

\begin{proposition}[Basin-depth-and-timescale]
\label{prop:basin-depth}
Under pure imitation the all-$\GOne$ state $j = 0$ is absorbing (\Cref{ssec:favoredness-bias}); posit in addition a small exploration rate $\lambda_{\mathrm{seed}} > 0$ at which an all-$\GOne$ community spontaneously trials a single $\GTwo$ binding, a mutation term reflecting the $j = 0$ boundary in the sense of the best-response dynamics with mutation of \citet{ellison_2000_basins}, so that all-$\GOne$ is metastable rather than absorbing. Let $\Fdelta$ satisfy monotonicity and endpoint inversion of \Cref{def:basin-conditions}, and write $\chistar := \inf\{X \in (0,1) : \Fdelta(X) \geq 0\}$ and $\cdepth := -\int_0^{\chistar} \Fdelta(Y)\, dY$; under endpoint-low and monotonicity $\Fdelta < 0$ on $(0, \chistar)$, so $\cdepth > 0$. When unsuccessful seeds resolve fast compared with the inter-seed clock $1/\lambda_{\mathrm{seed}}$, the mean time for the metastable all-$\GOne$ phase to transit to the all-$\GTwo$ absorbing state is, to leading order, that clock divided by the per-seed survival probability of \Cref{prop:ratio-favored}:
\begin{equation}
\E[\tau_{\mathrm{metastable}}] \;\approx\; \frac{1}{\lambda_{\mathrm{seed}} \cdot \rho_{\GTwo, \GOne}(\Nc, \beta)}.
\label{eq:metastable-identity}
\end{equation}
The single-seed survival probability admits the Lean-mechanized log ceiling
\begin{equation}
\log\!\sum_{k=1}^{\Nc-1} P_k \;\leq\; \frac{\Nc}{\eta_s}\, \cdepth \;+\; \frac{L}{\eta_s} \;+\; \log(\Nc-1),
\label{eq:fixation-log-ceiling}
\end{equation}
where $L$ is the Lipschitz constant of $\Fdelta$ and $\rho_{\GTwo,\GOne} = 1/(1 + \sum_k P_k)$ per the Karlin--Taylor identity \Cref{eq:7} (\texttt{fixation\_log\_rate\_cdepth\_le}; see \Cref{sec:formalization-scope}). Substituting \eqref{eq:fixation-log-ceiling} into \eqref{eq:metastable-identity} gives the policy-relevant large-$\Nc$ form
\begin{equation}
\E[\tau_{\mathrm{metastable}}] \;\sim\; \frac{1}{\lambda_{\mathrm{seed}}} \cdot \sqrt{\Nc} \cdot \exp\!\bigl(\cdepth\, \Nc / \eta_s\bigr),
\label{eq:kramers-asymptote}
\end{equation}
the classical Kramers asymptote for the discrete barrier at $\chistar$, valid when in addition $\Fdelta \in C^1$ near $\chistar$ with $\Fdelta'(\chistar) > 0$ (so the barrier is a nondegenerate quadratic maximum). The exponent $\cdepth\, \Nc / \eta_s$ and the ceiling \eqref{eq:fixation-log-ceiling} are mechanized; the $\sqrt{\Nc}$ discrete-Laplace prefactor in \eqref{eq:kramers-asymptote}, which alone needs the $C^1$ and strict-crossing hypotheses, is the classical Kramers leading-order form and stands outside the verified chain. The downstream achievability constraint of \Cref{ssec:fixation} consumes only \eqref{eq:fixation-log-ceiling}, which needs only the Lipschitz bound and the barrier location.
\end{proposition}

\begin{proof}[Sketch]
Equation \eqref{eq:metastable-identity} follows from three modelling conditions on the seeding process. (i) Seed arrivals are a Poisson point process of intensity $\lambda_{\mathrm{seed}}$, so inter-arrival times are i.i.d.\ exponential with mean $1/\lambda_{\mathrm{seed}}$. (ii) Unsuccessful seeds extinguish on a timescale short compared with $1/\lambda_{\mathrm{seed}}$, so between seed events the chain has returned to the all-$\GOne$ state and successive seeds face statistically identical conditions. (iii) Each seed therefore reaches the all-$\GTwo$ absorbing state independently with probability $\rho_{\GTwo,\GOne}(\Nc,\beta)$ of \Cref{prop:ratio-favored}, making the number of seeds until first success geometric with mean $1/\rho_{\GTwo,\GOne}$. Combining (i) and (iii) by Wald's identity gives \eqref{eq:metastable-identity}: the inter-seed waiting time $1/\lambda_{\mathrm{seed}}$ multiplied by the expected number of seeds $1/\rho_{\GTwo,\GOne}$ before one survives. Equation \eqref{eq:fixation-log-ceiling} is a Kramers-style discrete-barrier analysis on the Moran--Fermi birth--death chain: the discrete quasi-potential $V_{\Nc}(k/\Nc) = -\frac{1}{\Nc}\sum_{j=1}^{k}\Fdelta(j/\Nc)$ converges uniformly to $V(X) = -\int_0^X \Fdelta(Y)\,dY$ by the Lipschitz Riemann bound (\texttt{riemann\_sum\_error\_lipschitz}), and the running-barrier maximum at $\chistar$ caps $\log \sum_k P_k$ by the bound stated. The Kramers asymptote \eqref{eq:kramers-asymptote} follows from the discrete Laplace method applied to $\sum_k P_k$ at the saddle $\chistar$; the $\sqrt{\Nc}$ prefactor is the standard discrete-Laplace coefficient. Full derivation of \eqref{eq:fixation-log-ceiling} in \Cref{ssec:proof-basin-depth}.
\end{proof}

\begin{lemma}[Right-skew of the Hill density]
\label{lem:right-skew}
Let $\phi(\theta) := n K^n \theta^{n-1}/(K^n + \theta^n)^2$ be the Hill density with $n > 1$, $K > 0$, and peak at $\chidag := K\bigl((n-1)/(n+1)\bigr)^{1/n}$. Then
\begin{enumerate}
\item[(a)] $\phi(\chidag + \delta) > \phi(\chidag - \delta)$ for every $\delta \in (0, \chidag]$.
\item[(b)] If additionally $K \geq \Kstar(n)$, then $\phi(1 - X) \geq \phi(X)$ for every $X \in [0, 1/2]$, equivalently $g_i(X) := h_i(X) + h_i(1 - X)$ is non-decreasing on $[0, 1/2]$, where $h_i(X) := K^n/(K^n + X^n)$.
\end{enumerate}
\end{lemma}

\begin{proof}
\emph{Part (a): right-skew at the peak.} Set $t := \delta/\chidag \in (0, 1]$. Using $(\chidag)^n = K^n(n-1)/(n+1)$,
\[
K^n + (\chidag(1 \pm t))^n \;=\; \frac{K^n}{n + 1}\bigl[(n + 1) + (n - 1)(1 \pm t)^n\bigr].
\]
Define $A(t) := (n+1) + (n-1)(1+t)^n$, $B(t) := (n+1) + (n-1)(1-t)^n$, and the log-ratio
\[
f(t) \;:=\; \ln \frac{\phi(\chidag(1+t))}{\phi(\chidag(1-t))} \;=\; (n-1)\ln\frac{1+t}{1-t} + 2\ln\frac{B(t)}{A(t)}.
\]
Clearly $f(0) = 0$. Differentiating and collecting over the common denominator $(1-t^2) A(t) B(t)$, direct manipulation gives
\[
f'(t) \;=\; \frac{2(n^2 - 1)\, H(t)}{(1 - t^2)\, A(t)\, B(t)},
\]
where
\[
H(t) \;:=\; (n+1) + (n-1)\bigl[(1+t)^n + (1-t)^n\bigr] - n(1 - t^2)\bigl[(1+t)^{n-1} + (1-t)^{n-1}\bigr] - (n-1)(1 - t^2)^n.
\]
Direct evaluation gives $H(0) = (n+1) + 2(n-1) - 2n - (n-1) = 0$. Differentiating $H$ and grouping the $(1+t)^{n-1}$ and $(1-t)^{n-1}$ terms,
\[
H'(t) \;=\; n(n+1)\, t \bigl[(1+t)^{n-1} + (1-t)^{n-1}\bigr] + 2n(n-1)\, t\, (1 - t^2)^{n-1}.
\]
For $t \in (0, 1)$ and $n > 1$ every factor is positive, so $H'(t) > 0$. Combined with $H(0) = 0$, $H(t) > 0$ on $(0, 1)$; hence $f'(t) > 0$ on $(0, 1)$ and (since $f(0) = 0$) $f(t) > 0$ there. The endpoint $t = 1$ follows from $\phi(0) = 0 < \phi(2\chidag)$ directly.

\emph{Part (b): centro-symmetric monotonicity in Regime II.} The condition $K \geq \Kstar(n)$ is equivalent to $\chidag \geq 1/2$. Fix $X \in [0, 1/2]$, so $X \leq \chidag$ (rising tail). Two sub-cases for $1 - X$.

\emph{Sub-case (i): $1 - X \leq \chidag$.} Both $X$ and $1 - X$ lie on the rising tail $[0, \chidag]$ where $\phi$ is monotone non-decreasing. Since $X \leq 1/2 \leq 1 - X$, $\phi(X) \leq \phi(1 - X)$ immediately.

\emph{Sub-case (ii): $1 - X > \chidag$.} Set $\Delta_1 := \chidag - X \geq 0$ and $\Delta_2 := (1 - X) - \chidag > 0$. From $\chidag \geq 1/2 \geq X$ and $\chidag \geq 1/2$:
\[
\Delta_1 \;=\; \chidag - X \;\geq\; 1/2 - X, \qquad \Delta_2 \;=\; 1 - X - \chidag \;\leq\; 1/2 - X,
\]
so $\Delta_2 \leq \Delta_1$. By Part (a), $\phi(\chidag + \Delta_2) > \phi(\chidag - \Delta_2)$; by monotonicity of $\phi$ on the rising tail $[0, \chidag]$ together with $\chidag - \Delta_2 \geq \chidag - \Delta_1$, $\phi(\chidag - \Delta_2) \geq \phi(\chidag - \Delta_1)$. Chaining,
\[
\phi(1 - X) \;=\; \phi(\chidag + \Delta_2) \;>\; \phi(\chidag - \Delta_2) \;\geq\; \phi(\chidag - \Delta_1) \;=\; \phi(X).
\]

Equivalence with $g_i$ non-decreasing follows from $g_i'(X) = \phi(1 - X) - \phi(X)$ (since $h_i'(Y) = -\phi(Y)$).
\end{proof}

\begin{lemma}[Interior critical point in Regime III]
\label{lem:regime-iii-critical}
Let $\phi$, $\chidag$, $g_i$ be as in \Cref{lem:right-skew}, and assume $n > 1$ and $K < \Kstar(n)$ (so $\chidag < 1/2$). Then $g_i$ has an interior critical point in $(0, \chidag)$.
\end{lemma}

\begin{proof}
$g_i'(X) = \phi(1 - X) - \phi(X)$ (since $h_i'(Y) = -\phi(Y)$). Evaluating at the endpoints:
\[
g_i'(0) \;=\; \phi(1) - \phi(0) \;=\; \phi(1) \;>\; 0,
\]
since $\phi(0) = 0$ for $n > 1$ and $\phi(1) > 0$. The Hill density $\phi(\theta) = n K^n \theta^{n-1}/(K^n + \theta^n)^2$ satisfies $\phi'(\theta) \propto (n-1)K^n - (n+1)\theta^n$, so $\phi$ is strictly decreasing for $\theta > \chidag = K\bigl((n-1)/(n+1)\bigr)^{1/n}$. At $X = \chidag$ we have $1 - \chidag > \chidag$ (since $\chidag < 1/2$), so $\phi(1 - \chidag) < \phi(\chidag)$; hence
\[
g_i'(\chidag) \;=\; \phi(1 - \chidag) - \phi(\chidag) \;<\; 0.
\]
Continuity of $g_i'$ on $[0, \chidag]$ and the intermediate value theorem give $X_0 \in (0, \chidag)$ with $g_i'(X_0) = 0$.
\end{proof}

\subsection*{Detailed derivations}
\label{sec:omitted-proofs}

The full derivations of the results sketched above are collected here. The body and the statements earlier in this appendix retain only one-sentence interpretations.

\subsection{Proof of \texorpdfstring{\Cref{prop:favoredness-sufficient}}{the favoredness-sufficiency proposition}}
\label{ssec:proof-favoredness-sufficient}

CSP gives $\Fdelta(j/\Nc) + \Fdelta((\Nc - j)/\Nc) \geq 0$ for every $j$. By the strict clause of CSP, there exists $X_0 \in [0, 1/2]$ with $\Fdelta(X_0) + \Fdelta(1 - X_0) > 0$. By continuity of $\Fdelta$ (a finite affine combination of bounded-Lipschitz tail probabilities), the strict centro-symmetric inequality holds on an open neighborhood $(a, b) \subseteq [0, 1/2]$ of $X_0$. For $\Nc$ large enough that the grid step $1/\Nc$ is smaller than $b - a$, at least one grid point $j_0/\Nc$ lies in $(a, b)$, so the pair-sum $\Fdelta(j_0/\Nc) + \Fdelta((\Nc - j_0)/\Nc) > 0$ strictly. Every other pair is $\geq 0$ by CSP. When $\Nc$ is even, the unpaired middle term $\Fdelta(1/2) \geq 0$ by CSP at $X = 1/2$. The cumulative sum is therefore strictly positive, and the ratio of \Cref{eq:7} exceeds $1$. For fixed parameters with a strict-positive interval of width $w_* := b - a > 0$, the conclusion holds whenever $\Nc > 1/w_*$. The strict-positive interval can in principle be narrower than $1/\Nc$, so for small $\Nc$ with a tight strict-CSP margin the discrete sum of \Cref{prop:ratio-favored} should be checked directly rather than presumed from the large-$\Nc$ asymptotic. \qed

\subsection{Proof of \texorpdfstring{\Cref{thm:abc-preserved}}{the lifting theorem}}
\label{ssec:proof-abc-preserved}

The construction of $\bar U_g(X) := \E_{\theta}\bigl[\E_{\nu(\cdot \mid X, \theta)}[u_{\mathrm{perc}}^g(b)]\bigr]$ has two layers, inner over the fast-chain stationary distribution $\nu$ on the calibration state $b$ at fixed threshold $\theta$, outer over the threshold prior. Each layer is made exact by its own affine identity, and the atomic affine object is the per-genie perceived utility $u_{\mathrm{perc}}^g(b)$, not the payoff differential.

\paragraph{Fast-chain averaging is exact at fixed $g, \theta, \altvec$.}
The per-wisher perceived utilities of \Cref{eq:uperc-G2,eq:uperc-G1} are affine in $b$ for each binding:
\[
u_{\mathrm{perc}, i}^{\GTwo}(b) = u_{\mathrm{actual}}^{\GTwo} + \ualt{i}^{\GTwo} - \Bcap(1 - b),
\qquad
u_{\mathrm{perc}, i}^{\GOne}(b) = u_{\mathrm{actual}}^{\GOne} + \ualt{i}^{\GOne} + \Mcap - \Aa\, b,
\]
with $\ualt{i}^g := \altvec_i \cdot \mathbf{w}_g$ the wisher's altruistic-utility component under binding $g$ (\Cref{sec:welfare-utility}). Affine functions commute with expectation: for any probability measure $\nu$, $\E_\nu[u_{\mathrm{perc}, i}^g(b)] = u_{\mathrm{perc}, i}^g(\E_\nu[b])$.

The fast chain on $b$ is, under linear $R$ and linear $D$, itself affine: $b^{t+1} = (1 - L_b)\, b^t + \eta_b A_g \sigma_i^g(X) + \xi_t$, with linear relaxation rate $L_b := \eta_b A_g \sigma_i^g(X) + \mu\epsilon$ (so the Lipschitz constant of the deterministic part is $|1 - L_b|$, $< 1$ when $L_b \in (0, 2)$) and per-interaction noise $\xi_t$ centered by hypothesis. For the unconstrained affine chain, taking expectations under stationarity gives $m = (1 - L_b)\, m + \eta_b A_g \sigma$, hence $m = \eta_b A_g \sigma / L_b = b^*(X; \theta)$ with no bias from the noise. On a bounded $\mathcal{B}$, where reflection or truncation could otherwise displace it, the boundary-preserving stationary-mean assumption of \Cref{sec:wisher-dynamics} secures $\E_{\nu(\cdot \mid X, \theta)}[b] = b^*(X; \theta)$ directly; this identity, not the affine recursion as such, is what the rest of the proof uses.

Composing: $\E_{\nu(\cdot \mid X, \theta)}[u_{\mathrm{perc}}^g(b)] = u_{\mathrm{perc}}^g(b^*(X; \theta))$ exactly.

\paragraph{Threshold and altruism averaging are exact per binding.}
For each $g$, averaging over the threshold prior $\theta_i^g$ and the independent altruism prior $\altvec_i$,
\[
\bar U_g(X) \;=\; \E_{\theta, \altvec}\bigl[u_{\mathrm{perc}, i}^g(b^*(X; \theta_i^g), \altvec_i)\bigr] \;=\; \E_{\altvec}\!\bigl[u_{\mathrm{perc}, i}^g\bigl(\E_\theta[b^*(X; \theta_i^g)], \altvec_i\bigr)\bigr] \;=\; u_{\mathrm{perc}}^g\!\bigl(\bar\sigma_g(X), \altbar\bigr),
\]
where the second equality is the affineness of $u_{\mathrm{perc}}^g$ in $b$, and the third uses the sharp-threshold identity $b^*(X; \theta) = c_g\,\1[X \geq \theta]$ together with $\E_\theta[\1[X \geq \theta]] = 1 - \Ftbar{g}(X)$ to give $\bar\sigma_g(X) = c_g(1 - \Ftbar{g}(X))$, plus the linearity of $\ualt{i}^g = \altvec_i \cdot \mathbf{w}_g$ in $\altvec_i$ to give $\E_{\altvec}[\ualt{i}^g] = \altbar \cdot \mathbf{w}_g = W^g$. The difference $W^{\GTwo} - W^{\GOne} = \altbar \cdot \dHbar = \Wagg$ enters the payoff differential.

Subtracting the two per-binding identities reproduces \Cref{eq:3}:
\begin{align*}
\Fdelta(X) &= \bar U_{\GTwo}(X) - \bar U_{\GOne}(X) \\
&= u_{\mathrm{perc}}^{\GTwo}(\bar\sigma_{\GTwo}(X), \altbar) - u_{\mathrm{perc}}^{\GOne}(\bar\sigma_{\GOne}(X), \altbar) \\
&= (u_{\mathrm{actual}}^{\GTwo} - u_{\mathrm{actual}}^{\GOne}) + \Wagg - \Bcap - \Mcap + \Bcap\,\bar\sigma_{\GTwo}(X) + \Aa\,\bar\sigma_{\GOne}(X),
\end{align*}
identically.

\paragraph{Inheritance of the basin existence conditions.}
Because $\Fdelta$ equals the eq:3 form pointwise, each basin existence condition of \Cref{def:basin-conditions} transfers as a consequence of facts about $\bar\sigma_{\GTwo}, \bar\sigma_{\GOne}$ and the coefficients $\Bcap, \Aa$.

\emph{Monotonicity.} Both $\bar\sigma_{\GTwo}(X)$ and $\bar\sigma_{\GOne}(X)$ are non-decreasing in $X$, because the prior tail $\Ftbar{g}$ is non-increasing at each binding. The coefficients $\Bcap, \Aa \geq 0$. The composition is therefore non-decreasing in $X$, and $\Fdelta$ inherits the monotonicity pointwise.

\emph{Endpoint inversion.} Writing $u_w := (u_{\mathrm{actual}}^{\GTwo} - u_{\mathrm{actual}}^{\GOne}) + \Wagg$, at $X = 0$ both bindings have $\bar\sigma_g(0) = 0$, so
\begin{equation}
\Fdelta(0) = u_w - \Bcap - \Mcap.
\label{eq:lift-B-low}
\end{equation}
At $X = 1$, $\bar\sigma_g(1) = c_g(1 - \sigma_g^{\mathrm{sat}})$ per binding, so
\begin{equation}
\Fdelta(1) = u_w + \Aa\,\bar\sigma_{\GOne}(1) - \Bcap\,(1 - \bar\sigma_{\GTwo}(1)) - \Mcap.
\label{eq:lift-B-high}
\end{equation}

\emph{Centro-symmetric pairing (CSP).} For each $X \in [0, 1/2]$, direct substitution into the \Cref{eq:3} form gives the pair-sum
\begin{equation}
\Fdelta(X) + \Fdelta(1 - X) = 2(u_w - \Bcap - \Mcap) + \Bcap\bigl[\bar\sigma_{\GTwo}(X) + \bar\sigma_{\GTwo}(1 - X)\bigr] + \Aa\bigl[\bar\sigma_{\GOne}(X) + \bar\sigma_{\GOne}(1 - X)\bigr],
\label{eq:lift-C-sym}
\end{equation}
identical value-by-value to the corresponding pair-sum on the \Cref{eq:3} expression. A non-negative pair-sum on the pointwise form is therefore a non-negative pair-sum on $\Fdelta$, and strict positivity at any point in $[0, 1/2]$ transfers as strict positivity at the same point. \qed

\subsection{Proof of \texorpdfstring{\Cref{prop:basin-depth}}{the basin-depth proposition}}
\label{ssec:proof-basin-depth}

A Kramers-style barrier-crossing analysis on the Moran--Fermi birth--death chain of \Cref{ssec:favoredness-bias}, paralleling the stochastic-stability and escape-cost analysis for two-strategy logit-noise chains \citep[\S6--7]{sandholm_2010_stochastic} and the analogous treatment of best response with mutation in \citet{ellison_2000_basins}. The waiting-time scaling instantiates this logit-noise escape-cost analysis for the prior-induced gradient $\Fdelta$, recovering the classical mutation-selection results of \citet{foster_young_1990_stochastic} and \citet{kandori_mailath_rob_1993_learning} for $2\times 2$ coordination games as the uniform-mutation special case. What is specific to the present setting is finite-population and non-asymptotic. \Cref{eq:fixation-log-ceiling} bounds the barrier sum from the Lipschitz Riemann estimate and the running-maximum location alone, without the small-noise limit those results rest on.

\paragraph{Step 1: Mean first-passage on the discrete chain.}
For the slow chain on $j \in \{0, 1, \ldots, \Nc\}$ with up- and down-switch probabilities $T^+(j), T^-(j)$ of \Cref{ssec:favoredness-bias}, the leading-order mean-first-passage approximation for a birth--death chain gives the expected time to reach $j = \Nc$ starting from $j = 1$ as
\[
\E[\tau_{\mathrm{fix}}] \;=\; \sum_{k=1}^{\Nc - 1} \frac{1}{T^+(k)} \prod_{j=1}^{k} \frac{T^-(j)}{T^+(j)} \cdot \bigl(1 + o(1)\bigr),
\]
where the $(1+o(1))$ factor absorbs the inner geometric correction. Substituting the Fermi-rate ratio $T^-(j)/T^+(j) = \exp(-\Fdelta(j/\Nc)/\eta_s)$ of \Cref{eq:5} yields
\[
\E[\tau_{\mathrm{fix}}] \;=\; \sum_{k=1}^{\Nc - 1} \frac{1}{T^+(k)} \exp\!\left(-\frac{1}{\eta_s} \sum_{j=1}^{k} \Fdelta(j/\Nc)\right) \cdot \bigl(1 + o(1)\bigr),
\]
carrying only the $(1+o(1))$ correction noted above, the large-$\Nc$ reduction of the inner geometric series near the barrier.

\paragraph{Step 2: Discrete and continuous quasi-potential.}
Define the discrete quasi-potential as the running sum of selection gradients,
\[
V_{\Nc}(k/\Nc) \;:=\; -\frac{1}{\Nc} \sum_{j=1}^{k} \Fdelta(j/\Nc).
\]
Under the Lipschitz bound on $\Fdelta$ over $[0, 1]$ (following from the bounded density of the threshold prior; \Cref{eq:sigbar}), the right Riemann-sum error is $O(1/\Nc)$, so $V_{\Nc}$ converges uniformly to its continuous limit
\[
V(X) \;:=\; -\int_0^X \Fdelta(Y)\, dY, \qquad V_{\Nc}(k/\Nc) \;=\; V(k/\Nc) + O(1/\Nc).
\]
With $V(0) = 0$ and $V'(X) = -\Fdelta(X)$, monotonicity together with endpoint inversion gives $V$ strictly increasing on $(0, \chistar)$, attaining its maximum at $X = \chistar$, and strictly decreasing on $(\chistar, 1)$ (the interior $\chistar$ is the basin boundary of \Cref{thm:basin-boundary}). The basin depth is the barrier height,
\[
\cdepth \;=\; V(\chistar) - V(0) \;>\; 0.
\]

\paragraph{Step 3: Finite log-ceiling.}
The Karlin--Taylor terms are
\[
P_k \;=\; \prod_{j=1}^{k} T^-(j)/T^+(j) \;=\; \exp\bigl(\Nc\, V_{\Nc}(k/\Nc)/\eta_s\bigr),
\]
using $\sum_{j=1}^{k}\Fdelta(j/\Nc) = -\Nc\, V_{\Nc}(k/\Nc)$ and the Fermi-rate ratio of \Cref{eq:5}. Bounding the sum by its largest term times the number of terms,
\[
\sum_{k=1}^{\Nc-1} P_k \;\leq\; (\Nc-1)\,\max_{1 \leq k \leq \Nc-1} \exp\!\bigl(\Nc\, V_{\Nc}(k/\Nc)/\eta_s\bigr).
\]
The running maximum is controlled by the continuous barrier: $V_{\Nc}(k/\Nc) \leq V(\chistar) + L/\Nc = \cdepth + L/\Nc$, since $V$ attains its maximum $\cdepth$ at $\chistar$ (Step~2) and the right-Riemann error is bounded by $L/\Nc$ with $L$ the Lipschitz constant of $\Fdelta$ (\texttt{riemann\_sum\_error\_lipschitz}). Taking logarithms,
\[
\log\!\sum_{k=1}^{\Nc-1} P_k \;\leq\; \log(\Nc-1) + \frac{\Nc}{\eta_s}\Bigl(\cdepth + \frac{L}{\Nc}\Bigr) \;=\; \frac{\Nc}{\eta_s}\,\cdepth + \frac{L}{\eta_s} + \log(\Nc-1),
\]
which is \Cref{eq:fixation-log-ceiling}. This bound uses only the Lipschitz constant and the location of the running maximum at $\chistar$, not the $C^1$ or strict-crossing hypotheses; it is the only part of this proposition the achievability constraint of \Cref{ssec:fixation} consumes, and it is the form mechanized as \texttt{fixation\_log\_rate\_cdepth\_le}.

\paragraph{Step 4: Discrete Laplace approximation.}
The inner sum in Step 1 equals $\Nc V_{\Nc}(k/\Nc)$, so the product term is $\exp(\Nc V_{\Nc}(k/\Nc)/\eta_s)$. As $\Nc \to \infty$, the outer sum is dominated by terms near the barrier $k/\Nc = \chistar$. Expanding $V$ around $\chistar$,
\[
V(X) \;=\; \cdepth + \tfrac{1}{2} V''(\chistar)(X - \chistar)^2 + O\bigl((X - \chistar)^3\bigr),
\]
with $V''(\chistar) = -\Fdelta'(\chistar) < 0$. Applying the discrete Laplace method at grid spacing $1/\Nc$ against a continuous Gaussian of width $\sqrt{\eta_s/(\Nc |V''(\chistar)|)}$ gives
\[
\sum_{k=1}^{\Nc - 1} \exp\!\left(\frac{\Nc V_{\Nc}(k/\Nc)}{\eta_s}\right) \;\sim\; \sqrt{\frac{2\pi \Nc\, \eta_s}{|V''(\chistar)|}} \cdot \exp\!\left(\frac{\Nc \cdepth}{\eta_s}\right).
\]
The $T^+(k)^{-1}$ prefactor in Step 1 is $O(1)$ near the barrier. Combining,
\[
\E[\tau_{\mathrm{fix}}] \;\sim\; \sqrt{\Nc} \cdot \exp\!\left(\cdepth\, \Nc / \eta_s\right).
\]

\paragraph{Step 5: Selection-intensity inheritance.}
The selection intensity $\beta = 1/\eta_s$ enters the exponent inverse-linearly through the Fermi-rate ratio of \Cref{eq:5}, giving exponent $\Nc \cdepth / \eta_s$. Under the convention $\eta_s = \eta_b$ (\Cref{ssec:favoredness-bias}) this reads $\Nc \cdepth / \eta_b$, so the accessibility constraint \Cref{eq:10} is governed by the calibration-update step size, which enters both as the comparison temperature and, through $\rho_g$, in the barrier height $\cdepth$. \qed

\section{Threshold-prior shape class}
\label{sec:priors-shape}
\sloppy

The basin-existence framework rests on five prior-shape conditions, not on any specific parametric family: \emph{normalization}, the prior is a probability measure on $[0, \infty)$ or $[0, 1]$ with a density; \emph{tail regularity}, its tail $\Ftbari$ is continuous, monotone non-increasing, and bounded Lipschitz on $[0, 1]$; \emph{endpoint-low} and \emph{endpoint-high} of \Cref{def:basin-conditions} at $\Ftbari(0)$ and $\Ftbari(1)$; and \emph{centro-symmetric pairing} (CSP) at the prior's centro-symmetric maximum on $[0, 1/2]$. Normalization and tail regularity are intrinsic to the prior; endpoint-low, endpoint-high, and centro-symmetric pairing are the three conditions of \Cref{def:basin-conditions} read off at the values the prior supplies. Four families discharge them in the Lean~4 development: Hill (\texttt{hillPrior}, \texttt{Priors/Hill.lean}), Pareto Type~I (\texttt{paretoPrior}, \texttt{Priors/Pareto.lean}), Lomax / Pareto Type~II (\texttt{lomaxPrior}, \texttt{Priors/Lomax.lean}), and Fr\'echet / EV Type~II (\texttt{frechetPrior}, \texttt{Priors/Frechet.lean}). Each is a fully certified \texttt{ThresholdPrior} instance; normalization is handled in \texttt{Priors/Normalization.lean} (\texttt{hill\_integral\_eq\_one}, \texttt{pareto\_integral\_eq\_one}, \texttt{lomax\_integral\_eq\_one}, \texttt{frechet\_integral\_eq\_one}).

The four tails and the location of the centro-symmetric pair-sum maximum on $[0, 1/2]$ per family are as follows. Pareto, tail $(x_m/X)^\alpha$ with modal density at the scale $x_m$: the maximum is at $X = x_m$, proved by \texttt{paretoPairSum\_le\_at\_threshold}. Lomax, tail $(1 + X/\lambda)^{-\alpha}$ with monotone-decreasing density: the pair-sum is antitone, so the maximum is at $X = 0$, proved by \texttt{lomaxPairSum\_antitoneOn}. Hill, tail $K^n/(K^n + X^n)$, and Fr\'echet, tail $1 - e^{-(s/X)^\alpha}$, are the right-skew families: in Regime II the pair-sum is monotone on $[0, 1/2]$ (\texttt{hillPairSum\_monotoneOn}, \texttt{frechetPairSum\_monotoneOn}), so the maximum is at $X = 1/2$; in Regime III ($\chistar < 1/2$) it sits at an interior local maximizer, the critical point located by \texttt{regimeIII\_interior\_critical\_hill} and \texttt{regimeIII\_interior\_critical\_frechet}. In every figure the global $W_{\max}$ on $[0, 1/2]$ is then evaluated numerically, with the maximizer checked against a brute-force grid scan before any diagram is drawn.

These are three structurally distinct density shapes, monotone-decreasing (Lomax), scale-modal (Pareto), and right-skew unimodal (Hill, Fr\'echet), so the four families exercise the centro-symmetric-pairing condition across the range of cases it must handle.

The phase diagrams of \Cref{ssec:basin-diagrams} express these conditions through three dimensionless boundary functions. Dividing \Cref{eq:3} by $\Mcap$ and writing $\bar\sigma_g(X) = c_g(1 - \Ftbar{g}(X))$ gives the dimensionless payoff differential
\begin{equation}
\Fdelta(X)/\Mcap \;=\; \tilde w - b - 1 + \hat b\,\bigl(1 - \Ftbar{\GTwo}(X)\bigr) + \hat a\,\bigl(1 - \Ftbar{\GOne}(X)\bigr),
\label{eq:hill-Fdelta}
\end{equation}
in the controls of \Cref{app:notation}, with $\tilde w := d + u_d$, $a := \Aa/\Mcap$, $b := \Bcap/\Mcap$, and calibration-rescaled $\hat a := a\,c_{\GOne}$, $\hat b := b\,c_{\GTwo}$. Writing $g_i(X) := \Ftbar{i}(X) + \Ftbar{i}(1-X)$ and the pair-sum maximum $W_{\max} := \max_{X \in [0, 1/2]}\bigl[\hat b\, g_{\GTwo}(X) + \hat a\, g_{\GOne}(X)\bigr]$, the three conditions become
\[
\pilow := \tilde w - b - 1, \quad \pihigh := \pilow + \hat b\,(1 - \Ftbar{\GTwo}(1)) + \hat a\,(1 - \Ftbar{\GOne}(1)), \quad \pibind := 2\pilow + 2(\hat a + \hat b) - W_{\max},
\]
namely endpoint-low $\pilow < 0$, endpoint-high $\pihigh > 0$, and centro-symmetric pairing $\pibind \geq 0$ (strict at some $X$). The diagrams fix the working point $c_{\GOne} = c_{\GTwo} = 1$ (so $\hat a = a$, $\hat b = b$) and $d = 1$, and partition each scale plane by the zero level sets of $\pihigh$ and $\pibind$ for all four families.

\begin{table}[ht]
\centering
\small
\setlength{\tabcolsep}{6pt}
\renewcommand{\arraystretch}{1.5}
\begin{tabular}{@{}p{0.34\textwidth} c p{0.30\textwidth}@{}}
\toprule
Prior family, tail $\Ftbar{}(X)$ & $\chistar = \Ftbar{}^{-1}(v)$ (matched scale) & Lean shape-condition certificate \\
\midrule
Hill / log-logistic,\ \ $K^{n}/(K^{n}+X^{n})$ & $K\bigl((1-v)/v\bigr)^{1/n}$ & \texttt{hillPrior}; max at $\tfrac12$ or interior \\
Pareto Type~I,\ \ $(x_m/X)^{\alpha}$ on $X \ge x_m$ & $x_m\,v^{-1/\alpha}$ & \texttt{paretoPrior}; max at $x_m$ \\
Lomax / Pareto~II,\ \ $(1+X/\lambda)^{-\alpha}$ & $\lambda\,(v^{-1/\alpha} - 1)$ & \texttt{lomaxPrior}; max at $0$ \\
Fr\'echet / EV~II,\ \ $1 - e^{-(s/X)^{\alpha}}$ & $s\,\bigl(\ln\tfrac{1}{1-v}\bigr)^{-1/\alpha}$ & \texttt{frechetPrior}; max at $\tfrac12$ or interior \\
\bottomrule
\end{tabular}
\caption{Four threshold priors that satisfy the basin-existence conditions, spanning monotone-density (Lomax, Pareto) and unimodal-density (Hill, Fr\'echet) structure. Middle column: on the matched-scale slice the tipping point is the single inverse-tail value $\chistar = \Ftbar{}^{-1}(v)$ with $v = (u_d + d - 1 + a)/(a + b)$, defined when $v \in (\Ftbar{}(1), 1)$ so that a stable all-$\GTwo$ state exists. Right column: the Lean object certifying all five shape conditions, with the location of the centro-symmetric pair-sum maximum that each family's centro-symmetric-pairing lemma supplies; the full declaration list and the normalization route are in \Cref{sec:formalization-scope}. The shape index ($n$ or $\alpha$) is the architect's awareness-sharpness knob and the scale ($K$, $x_m$, $\lambda$, $s$) is unit choice; \Cref{fig:prior-grid} fixes $\alpha = 2$ to match Hill $n = 2$.}
\label{tab:prior-families}
\end{table}

\begin{table}[ht]
\centering
\small
\setlength{\tabcolsep}{8pt}
\renewcommand{\arraystretch}{1.3}
\begin{tabular}{@{}lccc@{}}
\toprule
Prior family & $\Ftbar{}(1)$ & $W_{\max}$ & $u_d^{\min} = 1 - a + W_{\max}/2$ \\
\midrule
Hill / log-logistic & $0.08$ & $1.19$ & $0.80$ \\
Pareto Type~I       & $0.04$ & $1.18$ & $0.79$ \\
Lomax / Pareto~II   & $0.18$ & $1.29$ & $0.85$ \\
Fr\'echet / EV~II   & $0.06$ & $1.18$ & $0.79$ \\
\bottomrule
\end{tabular}
\caption{Laundering threshold $u_d^{\min}$ across the four families, evaluated at the figure working point: a representative community (median acceptance threshold $0.3$; $a = 0.8$, $b = 0.3$; $c_{\GOne} = c_{\GTwo} = 1$; $\alpha = 2$ matched to Hill $n = 2$) on the $\Wagg = 0$ basin frontier ($d = 0$), where $u_d^{\min} = 1 - a + W_{\max}/2$. $W_{\max}$ is the pair-sum maximum located by each family's centro-symmetric-pairing maximizer (\Cref{tab:prior-families}); $\Ftbar{}(1)$ is the shared endpoint. The threshold is a field over the adoption plane, varying with the community (along the matched-scale diagonal it ranges over $[0.75, 1.30]$); at a fixed community it varies little across the prior because at the frontier it rides on $W_{\max}$, set by the perception levers and $\Ftbar{}(1)$ rather than the tail. The invariant is $u_d^{\min} > 0$ throughout: laundering requires the agent to be substantially, not marginally, more useful to the individual.}
\label{tab:udmin}
\end{table}

\section{Machine verification}
\label{sec:formalization-scope}
\sloppy

A machine-checked Lean~4 development against \texttt{mathlib} accompanies the analysis. This section separates the results derived from first principles inside Lean from the hypotheses retained at the boundary with the analytic asymptotics that frame them.

The development mechanizes the Moran--Fermi switch-rate ratio and the Karlin--Taylor ratio identity (\Cref{eq:5,eq:7}); the fixation probability $\phi_k$ of \Cref{eq:6}, whose absorbing-boundary harmonic recurrence is verified so that $\rho_{\GTwo,\GOne} = \phi_1$ and $\rho_{\GOne,\GTwo} = 1 - \phi_{\Nc-1}$ are obtained from it rather than posited as closed forms; the within-community ratio-favoredness criterion (\Cref{prop:ratio-favored}); the sufficiency of the structural conditions (\Cref{prop:favoredness-sufficient}), including the continuous-to-grid bridge and an explicit community-size threshold above which the cumulative sum is strictly positive; the affine lifting mechanism and the outer threshold-prior averaging $\bar\sigma_g(X) = c_g\bigl(1 - \Ftbar{g}(X)\bigr)$ underlying \Cref{thm:abc-preserved}; the finite-population accessibility rate bounds (the discrete-barrier log-rate form of \Cref{prop:basin-depth}); and \Cref{lem:right-skew} for the Hill density at general $n > 1$ in both its right-skew form (a) and its centro-symmetric monotonicity form (b), with the Regime~III interior-critical-point lemma (\Cref{lem:regime-iii-critical}) resting on the mechanized unimodality of the Hill density.

Three hypotheses are retained at the boundary, indexed (1)--(3) to match the table below. (1) Centro-symmetric pairing (CSP) is retained at the model capstone \texttt{g2\_ratio\_favored\_model}. Monotonicity of the model differential is derived from the monotonicity of $\bar\sigma_g$ (\texttt{modelFdelta\_condA}), and endpoint inversion reduces to the explicit control inequalities \Cref{eq:basin-B-low,eq:basin-B-high}. CSP is not derivable universally from the primitive parameters; it is hypothesised at the capstone. Per-family sufficient frontiers for CSP are mechanised via the centro-symmetric-pairing lemmas of \Cref{sec:priors-shape}, which exhibit the three structurally distinct cases CSP turns on. (2) The abstract capstone \texttt{accessible\_and\_favored} is stated for a $\Fdelta$ carrying continuity, a Lipschitz bound, and threshold-root regularity at $\chistar$. The model differential is shown continuous on $[0,1]$ via \texttt{modelFdelta\_continuousOn}; the global Lipschitz constant and the threshold-root regularity enter the abstract chain rather than being re-derived for the model. (3) The discrete-Laplace $\sqrt{\Nc}$ prefactor of \Cref{eq:kramers-asymptote} is hypothesised in \texttt{prop54\_log} as the classical leading-order form. The proved accessibility chain, \texttt{fixation\_log\_rate\_cdepth\_le} and its downstream consequences, does not consume it.

\begin{table}[ht]
\centering
\scriptsize
\renewcommand{\_}{\textunderscore\allowbreak}
\begin{tabular}{@{}p{1.9cm}p{3.2cm}p{1.1cm}p{1.3cm}p{3.4cm}@{}}
\toprule
Claim & Lean declaration(s) & Fully mech.? & Assump.\ level? & Notes \\
\midrule
\Cref{prop:ratio-favored} & \texttt{ratio\_favored\_iff}, \texttt{fixation\_recurrence} & Yes & No & Karlin--Taylor identity; $\phi_k$ (eq.~6) and its absorbing-boundary recurrence mechanized. \\
\Cref{prop:favoredness-sufficient} & \texttt{favoredness\_sufficient\_*} & Yes & No & Continuous-to-grid bridge; explicit floor $\Nc > 1/(\tfrac12-\chistar)$. \\
\Cref{thm:abc-preserved} & \texttt{affine\_commute\_integral}, \texttt{sigmabar\_eq\_integral}, \texttt{modelFdelta\_condA}, \texttt{modelFdelta\_condBhigh\_iff} & Partial & Partial & Inner/outer averaging and monotonicity derived; endpoint inversion explicit; CSP retained at the capstone (hypothesis 1). \\
\Cref{prop:basin-depth} & \texttt{accessible\_and\_favored} & Rate-form & Composition + $\sqrt{\Nc}$ & Finite ceiling and quasi-potential bridge derived; abstract continuity/Lipschitz/threshold-root inputs enter via composition (hypothesis 2); $\sqrt{\Nc}$ prefactor stated as classical form (hypothesis 3), unused in the chain. \\
\Cref{lem:right-skew}, \Cref{lem:regime-iii-critical} & \texttt{right\_skew\_general}, \texttt{right\_skew\_centro}, \texttt{regimeIII\_interior\_critical\_hill} & Yes & No & Both parts (a), (b); Regime III critical point, via Hill-density unimodality. \\
\bottomrule
\end{tabular}
\caption{Correspondence between manuscript claims and the Lean~4 development. ``Fully mechanized'' marks results proved from first principles; ``assumption-level'' marks results carrying an explicit analytic hypothesis in their statement. The three retained hypotheses are CSP at the model capstone, model-level accessibility composition, and the Kramers $\sqrt{\Nc}$ prefactor stated in classical leading-order form and not consumed downstream.}
\label{tab:formalization-map}
\end{table}

The development is \texttt{sorry}-free and introduces no custom axioms, depending only on the standard Lean/\texttt{mathlib} foundations \texttt{propext}, \texttt{Classical.choice}, and \texttt{Quot.sound}. The Lean sources and the phase-diagram scripts are available at \url{https://github.com/dlewissandy/two-genie-scripts} (commit \texttt{ebbd849}), built with Lean~\texttt{v4.29.0} and \texttt{mathlib}~\texttt{v4.29.0}; \texttt{lake build} re-checks the development, and the repository \textsc{readme} maps each paper result to its Lean declaration. The Lean declarations named in this paper reside under \texttt{formalization/}, the figure scripts under \texttt{figures/}.

\clearpage
\bibliographystyle{plainnat}
\bibliography{bibliography}

\typeout{get arXiv to do 4 passes: Label(s) may have changed. Rerun}
\end{document}